\documentclass{article}

\PassOptionsToPackage{round, compress}{natbib}
\usepackage[final]{neurips_2018}
\usepackage[utf8]{inputenc} 
\usepackage[T1]{fontenc}    
\usepackage{url}            
\usepackage{graphicx}       
\graphicspath{ {./images/} }
\usepackage{amsmath,eqnarray,easybmat}
\usepackage{blkarray}
\usepackage{enumerate,textpos}
\usepackage{algorithmic,algorithm} 
\usepackage{amssymb,amsthm,amsfonts,nicefrac}
\usepackage{color}
\usepackage{wrapfig}
\usepackage{csquotes}
\usepackage{titlesec}
\usepackage{multicol}
\usepackage{caption}
\usepackage{subcaption}
\usepackage{flushend}
\usepackage{listings}
\usepackage{float}
\usepackage{microtype}
\usepackage{hyperref}

\usepackage{MnSymbol}
\DeclareMathAlphabet\mathbb{U}{msb}{m}{n}
\usepackage{xpatch}

\let\P\relax 
\DeclareMathOperator*{\P}{\mathbb{P}}

\renewcommand{\hat}{\widehat}

\newcommand{\removed}[1]{}

\newcommand{\cA}{\mathcal{A}}

\newcommand{\cF}{\mathcal{F}}
\newcommand{\cP}{\mathcal{P}}

\newcommand{\cI}{\mathcal{I}}

\newcommand{\R}{\mathbb{R}}

\newcommand{\M}{\mathrm{M}}

\newcommand{\p}{\mathrm{p}}
\renewcommand{\P}{\mathrm{P}}

\newcommand{\U}{\mathrm{U}}

\newcommand{\x}{\mathrm{x}}

\newcommand{\norm}[1]{\left\| #1 \right\|}

\renewcommand{\log}[1]{\operatorname{log}\left(#1\right)}
\renewcommand{\exp}[1]{\operatorname{exp}\left(#1\right)}

\newcommand{\expectation}[2]{\mathbb{E}_{#1}\left[#2\right]}
\newcommand{\expect}[1]{\mathbb{E}\left[#1\right]}

\newcommand{\prob}[1]{\mathbb{P}\left[#1\right]}

\def\infinity{\rotatebox{90}{8}}

\newcommand{\set}[2][]{#1 \{ #2 #1 \} }

\newcommand{\ignore}[1]{}

\newtheorem{lemma}{Lemma}[section]

\newtheorem{theorem}[lemma]{Theorem}

\newtheorem{corollary}[lemma]{Corollary}

\theoremstyle{definition}
\newtheorem{definition}[lemma]{Definition}

\newcounter{note}[section]
\renewcommand{\thenote}{\thesection.\arabic{note}}
\newcommand{\mdnote}[1]{{\color{red}\refstepcounter{note}$\ll${\bf Mike's Comment \thenote: }{#1}$\gg$\marginpar{\tiny\bf MD \thenote}}}

\newcommand{\Aa}{\mathcal{A}}
\newcommand{\Uu}{\mathcal{U}}

\newcommand{\Ee}{\mathbb{E}}

\title{Policy Regret in Repeated Games}

\author{
  Raman Arora\\
  Dept. of Computer Science\\
  Johns Hopkins University\\
  Baltimore, MD 21204 \\
  \texttt{arora@cs.jhu.edu}
  \And
  Michael Dinitz\\
  Dept. of Computer Science\\
  Johns Hopkins University\\
  Baltimore, MD 21204 \\
  \texttt{mdinitz@cs.jhu.edu}
  \And
  Teodor V.~Marinov\\
  Dept. of Computer Science\\
  Johns Hopkins University\\
  Baltimore, MD 21204 \\
  \texttt{tmarino2@jhu.edu}
  \And
  Mehryar Mohri\\
  Courant Institute and Google Research\\
  New York, NY 10012 \\
  \texttt{mohri@cims.nyu.edu}
}

\begin{document}

\maketitle

\begin{abstract}

  The notion of \emph{policy regret} in online learning is a well defined performance measure for the common scenario of adaptive
  adversaries, which more traditional quantities such as external regret do not take into account. We revisit the notion of policy
  regret and first show that there are online learning settings in which policy regret and external regret are incompatible: any
  sequence of play that achieves a favorable regret with respect to one definition must do poorly with respect to the other. We then
  focus on the game-theoretic setting where the adversary is a self-interested agent. In that setting, we show that external regret and policy regret are not in conflict and, in fact, that a wide class of algorithms can ensure a favorable regret with respect to both definitions, so long as the adversary is also using such an algorithm. We also show that the sequence of play of no-policy
  regret algorithms converges to a \emph{policy equilibrium}, a new notion of equilibrium that we introduce. Relating this back to
  external regret, we show that coarse correlated equilibria, which no-external regret players converge to, are a strict subset of policy equilibria.  Thus, in game-theoretic settings, every sequence
  of play with no external regret also admits no policy regret, but the converse does not hold.

\end{abstract}

\section{Introduction}
\label{sec:intro}

Learning in dynamically evolving environments can be described as a
repeated game between a player, an online learning algorithm, and an
adversary. At each round of the game, the player selects an action,
e.g.\ invests in a specific stock, the adversary, which may be the
stock market, chooses a utility function, and the player gains the
utility value of its action. The player observes the utility value and
uses it to update its strategy for subsequent rounds. The player's
goal is to accumulate the largest possible utility over a finite
number of rounds of play.\footnote{Such games can be equivalently
  described in terms of minimizing losses rather than maximizing
  utilities. All our results can be equivalently expressed in terms of
  losses instead of utilities.}

The standard measure of the performance of a player is its
\emph{regret}, that is the difference between the utility achieved by
the best offline solution from some restricted class and the utility
obtained by the online player, when utilities are revealed
incrementally. Formally, we can model learning as the following
problem. Consider an action set $\Aa$. The player selects an action
$a_t$ at round $t$, the adversary picks a utility function $u_t$, and
the player gains the utility value $u_t(a_t)$. While in a full
observation setting the player observes the entire utility function
$u_t$, in a \emph{bandit setting} the player only observes the utility
value of its own action, $u_t(a_t)$. We use the shorthand $a_{1:t}$ to
denote the player's sequence of actions $(a_1, \ldots, a_t)$ and
denote by $\Uu_t = \set{ u\colon \Aa^t \to \R }$ the family of utility
functions $u_t$. The objective of the player is to maximize its
expected cumulative utility over $T$ rounds, i.e.\ maximize
$\Ee[\sum_{t = 1}^T f_t(a_{1:t})]$, where the expectation is
over the player’s (possible) internal randomization. Since this is clearly
impossible to maximize without knowledge of the future, the algorithm
instead seeks to achieve a performance comparable to that of the best
fixed action in hindsight.  Formally, \emph{external regret} is
defined as
\begin{equation}
\label{eq:regret}
R(T) =  \Ee\left[\max_{a \in \Aa} \sum_{t = 1}^T u_t(a_{1:t-1}, a) - \sum_{t = 1}^T u_t(a_{1:t})\right].
\end{equation}

A player is said to \emph{admit no external regret} if the external
regret is sublinear, that is $R(T) = o(T)$. In contrast to statistical learning, online learning algorithms do not need to make stochastic assumptions about data generation: strong regret bounds are possible even if the utility functions are adversarial.

There are two main adversarial settings in online learning: the
\emph{oblivious setting} where the adversary ignores the player's
actions and where the utility functions can be thought of as
determined before the game starts (for instance, in weather
prediction); and the \emph{adaptive setting} where the adversary can
react to the player's actions, thus seeking to throw the player off
track (e.g., competing with other agents in the stock market). More
generally, we define an \emph{$m$-memory bounded adversary} as one
that at any time $t$ selects a utility function based on the player's
past $m$ actions:
$u_t(a'_1, \ldots, a'_{t-m-1}, a_{t-m}, \ldots, a_t) = u_t(a_1, \ldots,
a_{t-m-1}, a_{t-m}, \ldots,  a_t)$, for all
$a'_1, \ldots, a'_{t - m -1}$ and all $a_1, \ldots, a_t$. An oblivious
adversary can therefore be equivalently viewed as a $0$-memory bounded
adversary and an adaptive adversary as an $\infinity$-memory bounded
adversary. For an oblivious adversary, external regret in
Equation~\ref{eq:regret} reduces to
$R(T) = \Ee[ \max_{a \in \Aa} \sum_{t = 1}^T u_t(a) - u_t(a_t)]$, since
the utility functions do not depend upon past actions.  Thus, external
regret is meaningful when the adversary is oblivious, but it does not
admit any natural interpretation when the adversary is adaptive. The
problem stems from the fact that in the definition of external regret,
the benchmark is a function of the player's actions.  Thus, if the
adversary is adaptive, or even memory-bounded for some $m > 0$, then,
external regret does not take into account how the adversary would
react had the player selected some other action.

To resolve this critical issue, \cite{arora2012online} introduced an
alternative measure of performance called \emph{policy regret}
for which the benchmark does not depend on the player's actions.
Policy regret is defined as follows
\begin{equation}
\label{eq:policyregret}
P(T) =  \max_{a \in \Aa} \sum_{t = 1}^T u_t(a, \ldots, a) - \Ee\left[\sum_{t = 1}^T u_t(a_{1:t})\right].
\end{equation}
\cite{arora2012online} further gave a reduction, using a mini-batch
technique where the minibatch size is larger than the memory $m$ of
adversary, that turns any algorithm with a sublinear external regret
against an oblivious adversary into an algorithm with a sublinear
policy regret against an $m$-memory bounded adversary, albeit at the
price of a somewhat worse regret bound, which is still sublinear.

In this paper, we revisit the problem of online learning against
adaptive adversaries. Since \cite{arora2012online} showed that there
exists an adaptive adversary against which any online learning
algorithm admits linear policy regret, even when the external regret
may be sublinear, we ask if no policy regret implies no external
regret. One could expect this to be the case since policy regret seems
to be a \emph{stronger} notion than external regret.  However, our
first main result (Theorem~\ref{thm_main:reg_incomp}) shows that this
in fact is \emph{not} the case and that the two notions of regret
are incompatible: there exist adversaries (or sequence of utilities)
on which action sequences with sublinear external regret admit linear
policy regret and action sequences with sublinear policy regret incur
linear external regret.

We argue, however, that such sequences may not arise in practical
settings, that is in settings where the adversary is a self-interested
entity. In such settings, rather than considering a malicious opponent
whose goal is to hurt the player by inflicting large regret, it seems
more reasonable to consider an opponent whose goal is to maximize his
own utility. In zero-sum games, maximizing one’s utility comes at the
expense of the other player’s, but there is a subtle difference
between an adversary who is seeking to maximize the player's regret
and an adversary who is seeking to minimize the player's utility (or
maximize his own utility).  We show that in such strategic game
settings there is indeed a strong relationship between external regret
and policy regret. In particular, we show in
Theorem~\ref{thm_main:stable} that a large class of \emph{stable}
online learning algorithms with sublinear external regret also benefit
from sublinear policy regret.

Further, we consider a two-player game where each player is playing a
no policy regret algorithm. It is known that no external regret play
converges to a coarse correlated equilibrium (CCE) in such a game, but
what happens when players are using no policy regret algorithms?  We
show in Theorem~\ref{thm_main:main_result} that the average play in
repeated games between no policy regret players converges to a
\emph{policy equilibrium}, a new notion of equilibrium that we
introduce.  Policy equilibria differ from more traditional notions of
equilibria such as Nash or CCEs in a crucial way.  Recall that a CCE
is defined to be a recommended joint strategy for players in a game
such that there is no incentive for any player to deviate unilaterally
from the recommended strategy if other players do not deviate.

What happens if the other players react to one player's deviation by
deviating themselves?  This type of reasoning is not captured by
external regret, but is essentially what is captured by policy regret.
Thus, our notion of policy equilibrium must take into account these
counterfactuals, and so the definition is significantly more complex.
But, by considering functions rather than just actions, we can define
such equilibria and prove that they exactly characterize no policy
regret play.

Finally, it becomes natural to determine the relationship between policy equilibria (which characterize no policy regret play) and CCEs (which characterize no external regret play).  We show in Theorems~\ref{lem_main:stable_ext_reg1} and~\ref{lem_main:stable_ext_reg2} that the set of CCEs is a strict subset of policy equilibria. In other words, every CCE can be thought of as a policy regret equilibrium, but no policy regret play might not converge to a CCE.

\section{Related work}
\label{sec:related_work}

The problem of minimizing policy regret in a fully adversarial setting
was first studied by~\cite{merhav2002sequential}. Their work dealt
specifically with the full observation setting and assumed that the
utility (or loss) functions were $m$-memory bounded. They gave regret
bounds in $O(T^{2/3})$. The follow-up work by
\cite{farias2006combining} designed algorithms in a reactive bandit
setting. However, their results were not in the form of regret bounds
but rather introduced a new way to compare against acting according to
a fixed expert strategy. \cite{arora2012online} studied $m$-memory
bounded adversaries both in the bandit and full information settings
and provided extensions to more powerful competitor classes considered
in swap regret and more general $\Phi$-regret. \cite{dekel2014bandits}
provided a lower bound in the bandit setting for switching cost
adversaries, which also leads to a tight lower bound for policy regret
in the order $\Omega(T^{2/3})$. Their results were later extended by
\cite{koren2017bandits} and \cite{koren2017multi}. More recently,
\cite{heidari2016tight} considered the multi-armed bandit problem
where each arm's loss evolves with consecutive pulls. The process
according to which the loss evolves was not assumed to be stochastic
but it was not arbitrary either -- in particular, the authors required
either the losses to be concave, increasing and to satisfy a
decreasing marginal returns property, or decreasing. The regret bounds
given are in terms of the time required to distinguish the optimal arm
from all others.

A large part of reinforcement learning is also aimed at studying
sequential decision making problems. In particular, one can define a
Markov Decision Process (MDP) by a set of states equipped with
transition distributions, a set of actions and a set of reward or loss
distributions associated with each state action pair. The transition
and reward distributions are assumed unknown and the goal is to play
according to a strategy that minimizes loss or maximizes reward. We
refer the reader to
\citep{sutton1998reinforcement,kakade2003sample,szepesvari2010algorithms}
for general results in RL. MDPs in the online setting with bandit
feedback or arbitrary payoff processes have been studied by
\cite{even2009online,yu2009markov,neu2010online} and
\cite{arora2012deterministic}.
\par

The tight connection between no-regret algorithms and correlated
equilibria was established and studied by
\cite{foster1997calibrated,fudenberg1999conditional,hart2000simple,BlumMansour2007}.
A general extension to games with compact, convex strategy sets was
given by \cite{stoltz2007learning}. No external regret dynamics were
studied in the context of socially concave
games~\citep{even2009convergence}. More recently,
\cite{hazan2008computational} considered more general notions of
regret and established an equivalence result between fixed-point
computation, the existence of certain no-regret algorithms, and the
convergence to the corresponding equilibria.  In a follow-up work by
\cite{mohri2014conditional} and \cite{MohriYang2017}, the authors
considered a more powerful set of competitors and showed that the
repeated play according to conditional swap-regret or transductive
regret algorithms leads to a new set of equilibria.

\section{Policy regret in reactive versus strategic environments}
\label{sec:react_strat}

Often, distant actions in the past influence an adversary more than
more recent ones.  The definition of policy
regret~\eqref{eq:policyregret} models this influence decay by assuming
that the adversary is $m$-memory bounded for some $m \in
\mathbb{N}$. This assumption is somewhat stringent, however, since ideally we
could model the current move of the adversary as a function of the
entire past, even if actions taken further in the past have less
significance. Thus, we extend the definition of \cite{arora2012online}
as follows.

\begin{definition}
\label{def:pol_regret}
The $m$-\emph{memory policy regret} at time $T$ of a sequence of actions $(a_t)_{t = 1}^T$ with respect to a fixed action $a$ in the action set $\cA$ and the sequence of utilities $(u_t)_{t = 1}^T$, where $u_t:\cA^t \rightarrow \mathbb{R}$ and $m \in \mathbb{N}\bigcup\{\infty\}$ is
\begin{align*}
P(T,a) = \sum_{t = 1}^T u_t(a_1,\cdots,a_{t-m},a,\cdots,a) - \sum_{t = 1}^T u_t(a_1,\cdots,a_t).
\end{align*}
We say that the sequence $(a_t)_{t = 1}^T$ has sublinear policy regret (or no policy regret) if $P(T,a) < o(T)$, for all actions $a\in\cA$.
\end{definition}

Let us emphasize that this definition is just an extension of the
standard policy regret definition and that, when the utility functions
are $m$-memory bounded, the two definitions exactly coincide.

While the motivation for policy regret suggests that this should be a
stronger notion compared to external regret, we show that not
only that these notions are incomparable in the general adversarial
setting, but that they are also \emph{incompatible} in a strong sense.

\begin{theorem}
\label{thm_main:reg_incomp}
There exists a sequence of $m$-memory bounded utility functions $(u_t)_{t = 1}^T$, where $u_t:\cA \rightarrow \mathbb{R}$, such that for any constant $m\geq 2$ (independent of $T$), any action sequence with sublinear policy regret will have linear external regret and any action sequence with sublinear external regret will have linear policy regret.
\end{theorem}

The proof of the above theorem constructs a sequence for which no reasonable play can attain sublinear external regret. In particular, the only way the learner can have sublinear external regret is if they choose to have very small utility. To achieve this, the utility functions chosen by the adversary are the following.  At time $t$, if the player chose to play the same action as their past $2$ actions then they get utility $\frac{1}{2}$. If the player's past two actions were equal but their current action is different, then they get utility $1$, and if their past two actions differ then no matter what their current action is they receive utility $0$. It is easy to see that the maximum utility play for this sequence (and the lowest $2$-memory bounded policy regret strategy) is choosing the same action at every round. However, such an action sequence admits linear
external regret. Moreover, every sublinear external regret strategy
must then admit sublinear utility and thus linear policy regret.
\par
As discussed in Section~\ref{sec:intro}, in many realistic environments we can instead think of the adversary as a self-interested agent trying to maximize their own utility, rather than trying to maximize the regret of the player. This more strategic environment is better captured by the game theory setting, in particular a $2$-player game where both players are trying to maximize their utility. Even though we have argued that external regret is not a good measure, our next result shows that minimizing policy regret in games can be done if both players choose their strategies according to certain no external regret algorithms. More generally, we adapt a classical notion of stability from the statistical machine learning setting and argue that if the players use no external regret algorithms that are \emph{stable}, then the players will have no policy regret in expectation. To state the result formally we first need to introduce some notation.

\paragraph{Game definition:} We consider a $2$-player game $\mathcal{G}$, with players $1$ and $2$. The action set of player $i$ is denoted by $\cA_i$, which we think of as being embedded into $\mathbb R^{|\cA_i|}$ in the obvious way where each action corresponds to a standard basis vector. The corresponding simplex is $\Delta \cA_i$. The action of player $1$ at time $t$ is $a_t$ and of player $2$ is $b_t$. The observed utility for player $i$ at time $t$ is $u_i(a_t,b_t)$ and this is a bi-linear form with corresponding matrix $\P_i$. We assume that the utilities are bounded in $[0,1]$. 

\paragraph{Algorithm of the player:} When discussing algorithms, we take the view of player $1$. Specifically, at time $t$, player $1$ plays according to an algorithm which can be described as $Alg_t : (\cA_1\times\cA_2)^t \rightarrow \Delta \cA_1$. We distinguish between two settings: full information, in which the player observes the full utility function at time $t$ (i.e., $u_1(\cdot,b_t)$), and the bandit setting, in which the player only observes $u_1(a_t,b_t)$. In the full information setting, algorithms like multiplicative weight updates (MWU~\cite{arora2012multiplicative}) depend only on the past $t-1$ utility functions $(u_1(\cdot,b_\ell))_{\ell=1}^{t-1}$, and thus we can think of $Alg_t$ as a function $f_t: \cA_2^t \rightarrow \Delta\cA_1$. In the bandit setting, though, the output at time $t$ of the algorithm depends both on the previous $t-1$ actions $(a_\ell)_{\ell=1}^{t-1}$ and on the utility functions (i.e., the actions picked by the other player).

But even in the bandit setting, we would like to think of the player's algorithm as a function $f_t: \cA_2^t \rightarrow \Delta\cA_1$. We cannot quite do this, however we \emph{can} think of the player's algorithm as a \emph{distribution} over such functions.  So how do we remove the dependence on $\cA_1^t$?  Intuitively, if we fix the sequence of actions played by player $2$, we want to take the expectation of $Alg_t$ over possible choices of the $t$ actions played by player $1$.  In order to do this more formally, consider the distribution $\mu$ over $\cA_1^{t-1} \times \cA_2^{t-1}$ generated by simulating the play of the players for $t$ rounds.  Then let $\mu_{b_{0:t}}$ be the distribution obtained by conditioning $\mu$ on the actions of player $2$ being $b_{0:t}$.  Now we let $f_t(b_{0:t-1})$ be the distribution obtained by sampling $a_{0:t-1}$ from $\mu_{b_{1:t-1}}$ and using $Alg(a_{0:t-1}, b_{0:t-1})$. When taking expectations over $f_t$, the expectation is taken with respect to the above distribution. We also refer to the output $p_t = f_t(b_{0:t-1})$ as the strategy of the player at time $t$.

Now that we can refer to algorithms simply as functions (or distributions over functions), we introduce the notion of a stable algorithm.

\begin{definition}
Let $f_t : \cA_2^t \rightarrow \Delta \cA_1$ be a sample from $Alg_t$ (as described above), mapping the past $t$ actions in $\cA_2$ to a distribution over the action set $\cA_1$. Let the distribution returned at time $t$ be $p_t^1 = f_t(b_1,\ldots,b_t)$. We call this algorithm \emph{on average} $(m,S(T))$ \emph{stable} with respect to the norm $\norm{\cdot}$, if for any $b_{t-m+1}',\ldots,b_t' \in \cA_2$ such that $\tilde p_t^1 = f_t(b_1,\ldots,b_{t-m},b_{t-m+1}',\ldots,b_t') \in \Delta\cA_1$, it holds that $\mathbb{E}[\sum_{t = 1}^T\norm{p_t^1 - \tilde p_t^1}] \leq S(T)$, where the expectation is taken with respect to the randomization in the algorithm.
\end{definition}

Even though this definition of stability is given with respect to the game setting, it is not hard to see that it can be extended to the general online learning setting, and in fact this definition is similar in spirit to the one given in~\cite{saha2012interplay}. It turns out that most natural no external regret algorithms are stable. In particular we show, in the supplementary, that both Exp3~\cite{auer2002nonstochastic} and MWU are on average $(m,m\sqrt{T})$ stable with respect to $\ell_1$ norm for any $m<o(\sqrt{T})$. It is now possible to show that if each of the players are facing stable no external regret algorithms, they will also have bounded policy regret (so the incompatibility from Theorem~\ref{thm_main:reg_incomp} cannot occur in this case).
\begin{theorem}
\label{thm_main:stable}
Let $(a_t)_{t = 1}^T$ and $(b_t)_{t = 1}^T$ be the action sequences of player $1$ and $2$ and suppose that they are coming from no external regret algorithms modeled by functions $f_t$ and $g_t$, with regrets $R_1(T)$ and $R_2(T)$ respectively. Assume that the algorithms are on average $(m,S(T))$ stable with respect to the $\ell_2$ norm. Then 
\begin{align*}
&\expect{P(T,a)} \leq \|\P_1\|S(T) + R_1(T)\\
&\expect{P(T,b)} \leq \|\P_2\|S(T) + R_2(T),
\end{align*}
where $u_t(a_{1:t})$ in the definition of $P(T,a)$ equals $u_1(a_t,g_t(a_{0:t-1}))$ and similarly in the definition of $P(T,b)$, equals $u_2(b_t,f_t(b_{0:t-1}))$. The above holds
for any fixed actions $b\in\cA_2$ and $a\in\cA_1$. Here the matrix norm $\|\cdot\|$ is the spectral norm.
\footnote{We would like to thank Mengxiao Zhang (USC) for suggesting how to improve on the above theorem and discovering a small error in one of our proofs, which has been corrected.}
\end{theorem}

\section{Policy equilibrium}
\label{sec:pol_eq}

\vspace*{-5pt}
Recall that unlike external regret, policy regret captures how other players in a game might react if a player decides to deviate from their strategy. The story is similar when considering different notions of equilibria. In particular Nash equlibria, Correlated equilibria and CCEs can be interpreted in the following way: if player $i$ deviates from the equilibrium play, their utility will not increase no matter how they decide to switch, provided that \emph{all other players continue to play according to the equilibrium}. This sentiment is a reflection of what no external and no swap regret algorithms guarantee. 
Equipped with the knowledge that no policy regret sequences are obtainable in the game setting under reasonable play from all parties, it is natural to
reason how other players would react if player $i$ deviated and what would be the cost of deviation when taking into account possible reactions. 
\par
Let us again consider the 2-player game setup through the view of player $1$. The player believes their opponent might be $m$-memory bounded and decides to proceed by playing according to a no policy regret algorithm. After many rounds of the game, player $1$ has computed an empirical distribution of play $\hat\sigma$ over $\cA:=\cA_1\times\cA_2$. The player is familiar with the guarantees of the algorithm and knows that, if instead, they changed to playing any fixed action $a\in\cA_1$, then the resulting empirical distribution of play $\hat\sigma_a$, where player $2$ has responded accordingly in a memory-bounded way, is such that $\expectation{(a,b)\sim\hat\sigma}{u_1(a,b)} \geq \expectation{(a,b)\sim\hat\sigma_a}{u_1(a,b)} - \epsilon$. This thought experiment suggests that if no policy regret play converges to an equilibrium, then the equilibrium is not only described by the deviations of player $1$, but also through the change in player $2$'s behavior, which is encoded in the distribution $\hat\sigma_a$. Thus, any equilibrium induced by no policy regret play, can be described by tuples of distributions $\{(\sigma,\sigma_a,\sigma_b): (a,b)\in\cA\}$, where $\sigma_a$ is the distribution corresponding to player $1$'s deviation to the fixed action $a\in\cA_1$ and $\sigma_b$ captures player $2$'s deviation to the fixed action $b\in\cA_2$. 
Clearly $\sigma_a$ and $\sigma_b$ are not arbitrary but we still need a formal way to describe how they arise. 

For convenience, lets restrict the memory of player $2$ to be $1$. Thus, what player $1$ believes is that at each round $t$ of the game, they play an action $a_t$ and player $2$ plays a function $f_t : \cA_1\rightarrow \cA_2$, mapping $a_{t-1}$ to $b_t = f_t(a_{t-1})$. Finally, the observed utility is $u_1(a_t,f_t(a_{t-1}))$. The empirical distribution of play, $\hat\sigma$, from the perspective of player $1$, 
is formed from the observed play $(a_t, f_t(a_{t-1}))_{t = 1}^T$.
Moreover, the distribution, $\hat \sigma_a$, that would have occurred if player $1$ chose to play action $a$ on every round is formed from the play $(a,f_t(a))_{t = 1}^T$. In the view of the world of player $1$, the actions taken by player $2$ are actually functions rather than actions in $\cA_2$. This suggests that the equilibrium induced by a no-policy regret play, is a distribution over the functional space defined below.

\begin{definition}
\label{def:func_sapce}
Let $\cF_1 := \{f:\cA_2^{m_1} \rightarrow \cA_1\}$ and $\cF_2 := \{g: \cA_1^{m_2} \rightarrow \cA_2\}$ denote the \emph{functional spaces of play} of players $1$ and $2$, respectively. Denote the product space by $\cF := \cF_1\times\cF_2$.
\end{definition}

Note that when $m_1=m_2=0$, $\cF$ is in a one-to-one correspondence with $\cA$, i.e.~when players believe their opponents are oblivious, we recover the action set studied in standard equilibria. For simplicity, for the remainder of the paper we assume that $m_1 = m_2 = 1$.  However, all of the definitions and results that follow can be extended to the fully general setting of arbitrary $m_1$ and $m_2$; see the supplementary for details. 

Let us now investigate how a distribution $\pi$ over $\cF$ can give rise to a tuple of distributions $(\hat\sigma,\hat\sigma_a,\hat\sigma_b)$. 
We begin by defining the utility of $\pi$ such that it equals the utility of a distribution $\sigma$ over $\cA$ i.e., we want $\expectation{(f,g)\sim\pi}{u_1(f,g)} = \expectation{(a,b)\sim\sigma}{u_1(a,b)}$. Since utilities are not defined for functions, we need an interpretation of $\expectation{(f,g)\sim\pi}{u_1(f,g)}$ which makes sense. We notice that $\pi$ induces a Markov chain with state space $\cA$ in the following way.
\begin{definition}
\label{def:func_proc}
Let $\pi$ be any distribution over $\cF$. Then \emph{$\pi$ induces a Markov process} with transition probabilities $\prob{(a_2,b_2) | (a_{1},b_{1})} = \sum_{(f,g)\in\cF_1\times\cF_2: f(b_{1}) = a_2, g(a_{1}) = b_{2}} \pi(f,g).$ We associate with this Markov process the transition matrix $\M \in \mathbb{R}^{|\cA|\times|\cA|}$, with $\M_{x_1,x_2} = \prob{x_2 | x_1}$ where $x_i = (a_i,b_i)$.
\end{definition}
Since every Markov chain with a finite state space has a stationary distribution, we think of utility of $\pi$ as the utility of a particular stationary distribution $\sigma$ of $\M$. How we choose $\sigma$ among all stationary distributions is going to become clear later, but for now we can think about $\sigma$ as the distribution which maximizes the utilities of both players. Next, we need to construct $\sigma_a$ and $\sigma_b$, which capture the deviation in play, when player $1$ switches to action $a$ and player $2$ switches to action $b$. The no-policy regret guarantee can be interpreted as $\expectation{(f,g)\sim\pi}{u_1(f,g)} \geq \expectation{(f,g)\sim\pi}{u_1(a,g(a))}$ i.e., if player $1$ chose to switch to a fixed action (or equivalently, the constant function which maps everything to the action $a\in\cA_1$), then their utility should not increase. Switching to a fixed action $a$, changes $\pi$ to a new distribution $\pi_a$ over $\cF$. This turns out to be a product distribution which also induces a Markov chain.
\begin{definition}
\label{def:proc_dev}
Let $\pi$ be any distribution over $\cF$. Let $\delta_a$ be the distribution over $\cF_1$ putting all mass on the constant function mapping all actions $b\in\cA_2$ to the fixed action $a\in\cA_1$. Let $\pi_{\cF_2}$ be the marginal of $\pi$ over $\cF_2$. The \emph{distribution resulting from player 1 switching to playing a fixed action $a \in \Aa$}, is denoted as $\pi_a = \delta_a\times\pi_{\cF_2}$. This distribution induces a Markov chain with transition probabilities $\prob{(a,b_2)|(a_1,b_1)} = \sum_{(f,g):g(a_1) = b_2}\pi(f,g)$ and \emph{the transition matrix of this Markov process is denoted by $\M_a$}. The distribution $\pi_b$ and matrix $\M_b$ are defined similarly for player $2$.
\end{definition}
Since the no policy regret algorithms we work with do not directly induce distributions over the functional space $\cF$ but rather only distributions over the action space $\cA$, we would like to state all of our utility inequalities in terms of distributions over $\cA$. Thus, we would like to check if there is a stationary distribution $\sigma_a$ of $\M_a$ such that $\expectation{(f,g)\sim\pi}{u_1(a,g(a))} = \expectation{(a,b)\sim\sigma_a}{u_1(a,b)}$. 
This is indeed the case as verified by the following theorem.

\begin{theorem}
\label{thm_main:pol_reg_cons}
Let $\pi$ be a distribution over the product of function spaces $\cF_1\times\cF_2$. There exists a stationary distribution $\sigma_a$ of the Markov chain $\M_a$ for any fixed $a\in\cA_1$ such that $\expectation{(a,b)\sim\sigma_a}{u_1(a,b)} = \expectation{(f,g)\sim\pi}{u_1(a,g(a))}$. Similarly, for every fixed action $b\in\cA_2$, there exists a stationary distribution $\sigma_b$ of $\M_b$ such that $\expectation{(a,b)\sim\sigma_b}{u_2(a,b)} = \expectation{(f,g)\sim\pi}{u_2(f(b),b)}$. 
\end{theorem}
The proof of this theorem is constructive and can be found in the supplementary. 
With all of this notation we are ready to formally describe what no-policy regret play promises in the game setting in terms of an equilibrium. 
\begin{definition}
\label{def:pol_eq}
A distribution $\pi$ over $\cF_1\times\cF_2$ is a \emph{policy equilibrium} if for all fixed actions $a\in\mathcal{A}_1$ and $b\in\cA_2$, which generate Markov chains $\M_a$ and $\M_b$ respectively, with stationary distributions $\sigma_a$ and $\sigma_b$ from Theorem~\ref{thm_main:pol_reg_cons}, there exists a stationary distribution $\sigma$ of the Markov chain $\M$ induced by $\pi$ such that:
\begin{equation}
\label{eq:yes}
\begin{aligned}
\expectation{(a,b)\sim\sigma}{u_1(a,b)} &\geq \expectation{(a,b)\sim\sigma_a}{u_1(a,b)}\\
\expectation{(a,b)\sim\sigma}{u_2(a,b)} &\geq \expectation{(a,b)\sim\sigma_b}{u_2(a,b)}.
\end{aligned}
\end{equation}
\end{definition}
In other words, $\pi$ is a policy equilibrium if there exists a stationary distribution $\sigma$ of the Markov chain corresponding to $\pi$, such that, when actions are drawn according to $\sigma$, no player has incentive to change their action. For a simple example of a policy equilibrium see Section~\ref{sec:simp_example} in the supplementary.

\vspace*{-3pt}
\subsection{Convergence to the set of policy equilibria}
\vspace*{-2pt}
We have tried to formally capture the notion of equilibria in which player $1$'s deviation would lead to a reaction from player $2$ and vice versa in Definition~\ref{def:pol_eq}. This definition is inspired by the counter-factual guarantees of no policy regret play and we would like to check that if players' strategies yield sublinear policy regret then the play converges to a policy equilibrium. Since the definition of sublinear policy regret does not include a distribution over functional spaces but only works with empirical distributions of play, we would like to present our result in terms of distributions over the action space $\cA$. Thus we begin by defining the set of all product distributions $\sigma\times\sigma_a\times\sigma_b$, induced by policy equilibria $\pi$ as described in the previous subsection. Here $\sigma_a$ and $\sigma_b$ represent the deviation in strategy if player $1$ changed to playing the fixed action $a\in\cA_1$ and player $2$ changed to playing the fixed action $b\in\cA_2$ respectively as constructed in Theorem~\ref{thm_main:pol_reg_cons}.

\begin{definition}
\label{def:pol_eq_set}
For a policy equilibrium $\pi$, let $S_\pi$ be the set of all stationary distributions which satisfy the equilibrium inequalities~\eqref{eq:yes}, 
$S_{\pi}  := \{\sigma\times\sigma_a\times\sigma_b: (a,b)\in\cA\}$ \removed{\mdnote{this notation doesn't really work -- what is the union of distributions?}}. Define $S = \bigcup_{\pi\in\Pi} S_{\pi }$, where $\Pi$ is the set of all policy equilibria.
\end{definition}
Our main result states that the sequence of empirical product distributions formed after $T$ rounds of the game $\hat\sigma\times\hat\sigma_a\times\hat\sigma_b$ is going to converge to $S$. Here $\hat\sigma_a$ and $\hat\sigma_b$ denote the distributions of deviation in play, when player $1$ switches to the fixed action $a\in\cA_1$ and player $2$ switches to the fixed action $b\in\cA_2$ respectively. We now define these distributions formally.
\begin{definition}
\label{def:emp_distr_dev}
Suppose player $1$ is playing an algorithm with output at time $t$ given by $f_t:\cA_2^t\rightarrow \Delta\cA_1$ i.e. $p_t^1 = f_t(b_{0:t-1})$. Similarly, suppose player $2$ is playing an algorithm with output at time $t$ given by $p_t^2 = g_t(a_{0:t-1})$. The empirical distribution at time $T$ is $\hat\sigma := \frac{1}{T}\sum_{t = 1}^T p_t$, where $p_t = p_t^1\times p_t^2$ is the product distribution over $\cA$ at time $t$. Further let $(p^2_a)_t = g_t(a_{0:t-m},a, \ldots,a)$ denote the distribution at time $t$, provided that player $1$ switched their strategy to the constant action $a\in\cA_1$. Let $\delta_a$ denote the distribution over $\cA_1$ which puts all the probability mass on action $a$. Let $(p_a)_t = \delta_a\times(p^2_a)_t$ be the product distribution over $\cA$, corresponding to the change of play at time $t$. Denote by $\hat\sigma_a = \frac{1}{T}\sum_{t = 1}^T (p_a)_t$ the empirical distribution corresponding to the change of play. The distribution $\hat\sigma_b$ is defined similarly.
\end{definition}

Suppose that $f_t$ and $g_t$ are no-policy regret algorithms, then our main result states that the sequence $(\hat\sigma\times\hat\sigma_a\times\hat\sigma_b)_T$ converges to the set $S$.\\ \vspace*{-9pt}

\begin{theorem}
\label{thm_main:main_result}
If the algorithms played by player $1$ in the form of $f_t$ and player $2$ in the form of $g_t$ give sub-linear policy regret sequences, then the sequence of product distributions $(\hat\sigma\times\hat\sigma_a\times\hat\sigma_b)_{T=1}^\infty$ converges weakly to the set $S$.
\end{theorem}

In particular if both players are playing MWU or Exp3, we know that they will have sublinear policy regret. Not surprisingly, we can show something slightly stronger as well. Let $\tilde\sigma$, $\tilde\sigma_a$ and $\tilde\sigma_b$ denote the empirical distributions of observed play corresponding to $\hat\sigma$, $\hat\sigma_a$ and $\hat\sigma_b$, i.e.\ $\tilde\sigma = \frac{1}{T} \delta_t$, where $\delta_t$ denotes the Dirac distribution, putting all weight on the played actions at time $t$. Then these empirical distributions also converge to $S$ almost surely.

\vspace*{-3pt}
\subsection{Sketch of proof of the main result}
\vspace*{-3pt}

The proof of Theorem~\ref{thm_main:main_result} has three main steps. The first step defines the natural empirical Markov chains $\hat\M$, $\hat\M_a$ and $\hat\M_b$ from the empirical play $(p_t)_{t = 1}^\infty$ (see Definition~\ref{def:emp_proc}) and shows that the empirical distributions $\hat\sigma$, $\hat\sigma_a$ and $\hat\sigma_b$ are stationary distributions of the respective Markov chains. The latter is done in Lemma~\ref{lem:stationary}. The next step is to show that the empirical Markov chains converge to Markov chains $\M$, $\M_a$ and $\M_b$ induced by some distribution $\pi$ over $\cF$. In particular, we construct an empirical distribution $\hat\pi$ and distributions $\hat\pi_a$ and $\hat\pi_b$ corresponding to player's deviations (see Definition~\ref{def:emp_func_distr}), and show that these induce the Markov chains $\hat\M$, $\hat\M_a$ and $\hat\M_b$ respectively (Lemma~\ref{lem:emp_func_distr}). The distribution $\pi$ we want is now the limit of the sequence $(\hat\pi)_T$. The final step is to show that $\pi$ is a policy equilibrium. The proof goes by contradiction. Assume $\pi$ is not a policy equilibrium, this implies that no stationary distribution of $\M$ and corresponding stationary distributions of $\M_a$ and $\M_b$ can satisfy inequalities~\eqref{eq:yes}. Since the empirical distributions $\hat\sigma$, $\hat\sigma_a$ and $\hat\sigma_b$ of the play satisfies inequalities~\eqref{eq:yes} up to an $o(1)$ additive factor, we can show, in Theorem~\ref{thm:conv_thm}, that in the limit, the policy equilibrium inequalities are exactly satisfied. Combined with the convergence of $\hat\M$, $\hat\M_a$ and $\hat\M_b$ to  $\M$, $\M_a$ and $\M_b$, respectively, this implies that stationary distributions of $\M$, $\M_a$ and $\M_b$, satisfying~\eqref{eq:yes}, giving a contradiction.

We would like to emphasize that the convergence guarantee of Theorem~\ref{thm_main:main_result} does not rely on there being a unique stationary distribution of the empirical Markov chains $\hat\M$, $\hat\M_a$ and $\hat\M_b$ or their respective limits $\M, \M_a, \M_b$. Indeed, Theorem~\ref{thm_main:main_result} shows that any limit point of $\{(\hat\sigma,\hat\sigma_a,\hat\sigma_b)_T\}_{T=1}^\infty$ satisfies the conditions of Definition~\ref{def:pol_eq}.   The proof does not require that any of the respective Markov chains have a unique stationary distribution, but rather requires only that $\hat\sigma$ has sublinear policy regret. We would also like to remark that $\{(\hat\sigma,\hat\sigma_a,\hat\sigma_b)_T\}_{T=1}^\infty$ need not have a unique limit and our convergence result only guarantees that the sequence is going to the set $S$. This is standard when showing that any type of no regret play converges to an equilibrium, see for example~\cite{stoltz2007learning}.

\vspace*{-3pt}
\subsection{Relation of policy equlibria to CCEs}
\vspace*{-2pt}
So far we have defined a new class of equilibria and shown that they correspond to no policy regret play. Furthermore, we know that if both players in a 2-player game play stable no external regret algorithms, then their play also has sublinear policy regret. It is natural to ask if every CCE is also a policy equilibrium: if $\sigma$ is a CCE, is there a corresponding policy equilibrium $\pi$ which induces a Markov chain $\M$ for which $\sigma$ is a stationary distribution satisfying~\eqref{eq:yes}? We show that the answer to this question is positive:

\begin{theorem}
\label{lem_main:stable_ext_reg1}
For any CCE $\sigma$ of a 2-player game $\mathcal{G}$, there exists a policy-equilibrium $\pi$ which induces a Markov chain $\M$ with stationary distribution $\sigma$.
\end{theorem}
To prove this, we show that for any CCE we can construct stable no-external regret algorithm which converge to it, and so since stable no-external regret algorithms always converge to policy equilibria (Theorem~\ref{thm_main:stable}), this implies the CCE is also a policy equilibrium.

However, we show the converse is not true: policy equilibria can give rise to behavior which is not a CCE.    
Our proof appeals to a utility sequence which is similar in spirit to the one in Theorem~\ref{thm_main:reg_incomp}, but is adapted to the game setting. 

\begin{theorem}
\label{lem_main:stable_ext_reg2}
There exists a 2-player game $\mathcal{G}$ and product distributions $\sigma\times\sigma_a\times\sigma_b \in S$ (where $S$ is defined in Definition~\ref{def:pol_eq_set} as the possible distributions of play from policy equilibria), such that $\sigma$ is not a CCE of $\mathcal{G}$.
\end{theorem}

In Section~\ref{sec:simp_example} of the supplementary we give a simple example of a policy equilibrium which is not a CCE.

\section{Discussion}
In this work we gave a new twist on policy regret by examining it in the game setting, where we introduced the notion of policy equilibrium and showed that it captures the behavior of no policy regret players.  While our characterization is precise, we view this as only the first step towards truly understanding policy regret and its variants in the game setting.  Many interesting open questions remain. Even with our current definitions, since we now have a broader class of equilibria to consider it is natural to go back to the extensive literature in algorithmic game theory on the price of anarchy and price of stability and reconsider it in the context of policy equilibria.  For example~\cite{roughgarden2015intrinsic} showed that in ``smooth games'' the worst CCE is no worse than the worst Nash.  Since policy equilibria contain all CCEs (Theorem~\ref{lem_main:stable_ext_reg1}), is the same true for policy equilibria? 

Even more interesting questions remain if we change our definitions to be more general.  For example, what happens with more than 2 players?  With three or more players, definitions of ``reaction'' by necessity become more complicated.  Or what happens when $m$ is not a constant?  No policy regret algorithms exist for superconstant $m$, but our notion of equilibrium requires $m$ to be constant in order for the Markov chains to make sense.  Finally, what if we compare against deviations that are more complicated than a single action, in the spirit of swap regret or $\Phi$-regret? 

From an online learning perspective, note that our notion of on average stable and the definition of $m$-memory boundedness are different notions of stability.  Is there one unified definition of ``stable'' which would allow us to give no policy regret algorithms against stable adversaries even outside of the game setting?

\subsubsection*{Acknowledgments}

This work was supported in part by NSF BIGDATA grant IIS-1546482, NSF BIGDATA grant IIS-1838139, NSF CCF-1535987, NSF IIS-1618662, NSF CCF-1464239, and NSF AITF CCF-1535887.

\bibliographystyle{plainnat}
\bibliography{mybib}

\begin{thebibliography}{28}
\providecommand{\natexlab}[1]{#1}
\providecommand{\url}[1]{\texttt{#1}}
\expandafter\ifx\csname urlstyle\endcsname\relax
  \providecommand{\doi}[1]{doi: #1}\else
  \providecommand{\doi}{doi: \begingroup \urlstyle{rm}\Url}\fi

\bibitem[Allen-Zhu and Li(2016)]{allen2016lazysvd}
Zeyuan Allen-Zhu and Yuanzhi Li.
\newblock Lazysvd: Even faster svd decomposition yet without agonizing pain.
\newblock In \emph{Advances in Neural Information Processing Systems}, pages
  974--982, 2016.

\bibitem[Arora et~al.(2012{\natexlab{a}})Arora, Dekel, and
  Tewari]{arora2012deterministic}
Raman Arora, Ofer Dekel, and Ambuj Tewari.
\newblock Deterministic {MDPs} with adversarial rewards and bandit feedback.
\newblock In \emph{Proceedings on Uncertainty in Artificial Intellegence
  (UAI)}, 2012{\natexlab{a}}.

\bibitem[Arora et~al.(2012{\natexlab{b}})Arora, Dekel, and
  Tewari]{arora2012online}
Raman Arora, Ofer Dekel, and Ambuj Tewari.
\newblock Online bandit learning against an adaptive adversary: from regret to
  policy regret.
\newblock In \emph{Proceedings of International Conference on Machine Learning
  (ICML)}, 2012{\natexlab{b}}.

\bibitem[Arora et~al.(2012{\natexlab{c}})Arora, Hazan, and
  Kale]{arora2012multiplicative}
Sanjeev Arora, Elad Hazan, and Satyen Kale.
\newblock The multiplicative weights update method: a meta-algorithm and
  applications.
\newblock \emph{Theory of Computing}, 8\penalty0 (1):\penalty0 121--164,
  2012{\natexlab{c}}.

\bibitem[Auer et~al.(2002)Auer, Cesa-Bianchi, Freund, and
  Schapire]{auer2002nonstochastic}
Peter Auer, Nicolo Cesa-Bianchi, Yoav Freund, and Robert~E Schapire.
\newblock The nonstochastic multiarmed bandit problem.
\newblock \emph{SIAM journal on computing}, 32\penalty0 (1):\penalty0 48--77,
  2002.

\bibitem[Blum and Mansour(2007)]{BlumMansour2007}
Avrim Blum and Yishay Mansour.
\newblock From external to internal regret.
\newblock \emph{Journal of Machine Learning Research}, 8:\penalty0 1307--1324,
  2007.

\bibitem[Dekel et~al.(2014)Dekel, Ding, Koren, and Peres]{dekel2014bandits}
Ofer Dekel, Jian Ding, Tomer Koren, and Yuval Peres.
\newblock Bandits with switching costs: {$T^{2/3}$} regret.
\newblock In \emph{Proceedings of the forty-sixth annual ACM symposium on
  Theory of computing}, pages 459--467. ACM, 2014.

\bibitem[Even-Dar et~al.(2009{\natexlab{a}})Even-Dar, Kakade, and
  Mansour]{even2009online}
Eyal Even-Dar, Sham~M Kakade, and Yishay Mansour.
\newblock Online markov decision processes.
\newblock \emph{Mathematics of Operations Research}, 34\penalty0 (3):\penalty0
  726--736, 2009{\natexlab{a}}.

\bibitem[Even-Dar et~al.(2009{\natexlab{b}})Even-Dar, Mansour, and
  Nadav]{even2009convergence}
Eyal Even-Dar, Yishay Mansour, and Uri Nadav.
\newblock On the convergence of regret minimization dynamics in concave games.
\newblock In \emph{Proceedings of the forty-first annual ACM symposium on
  Theory of computing}, pages 523--532. ACM, 2009{\natexlab{b}}.

\bibitem[Farias and Megiddo(2006)]{farias2006combining}
Daniela Pucci~De Farias and Nimrod Megiddo.
\newblock Combining expert advice in reactive environments.
\newblock \emph{Journal of the ACM (JACM)}, 53\penalty0 (5):\penalty0 762--799,
  2006.

\bibitem[Foster and Vohra(1997)]{foster1997calibrated}
Dean~P Foster and Rakesh~V Vohra.
\newblock Calibrated learning and correlated equilibrium.
\newblock \emph{Games and Economic Behavior}, 21\penalty0 (1-2):\penalty0 40,
  1997.

\bibitem[Fudenberg and Levine(1999)]{fudenberg1999conditional}
Drew Fudenberg and David~K Levine.
\newblock Conditional universal consistency.
\newblock \emph{Games and Economic Behavior}, 29\penalty0 (1-2):\penalty0
  104--130, 1999.

\bibitem[Hart and Mas-Colell(2000)]{hart2000simple}
Sergiu Hart and Andreu Mas-Colell.
\newblock A simple adaptive procedure leading to correlated equilibrium.
\newblock \emph{Econometrica}, 68\penalty0 (5):\penalty0 1127--1150, 2000.

\bibitem[Hazan and Kale(2008)]{hazan2008computational}
Elad Hazan and Satyen Kale.
\newblock Computational equivalence of fixed points and no regret algorithms,
  and convergence to equilibria.
\newblock In \emph{Advances in Neural Information Processing Systems}, pages
  625--632, 2008.

\bibitem[Heidari et~al.(2016)Heidari, Kearns, and Roth]{heidari2016tight}
Hoda Heidari, Michael Kearns, and Aaron Roth.
\newblock Tight policy regret bounds for improving and decaying bandits.
\newblock In \emph{IJCAI}, pages 1562--1570, 2016.

\bibitem[Kakade et~al.(2003)]{kakade2003sample}
Sham~Machandranath Kakade et~al.
\newblock \emph{On the sample complexity of reinforcement learning}.
\newblock PhD thesis, University of London London, England, 2003.

\bibitem[Koren et~al.(2017{\natexlab{a}})Koren, Livni, and
  Mansour]{koren2017bandits}
Tomer Koren, Roi Livni, and Yishay Mansour.
\newblock Bandits with movement costs and adaptive pricing.
\newblock In \emph{Proceedings of the 2017 Conference on Learning Theory},
  pages 1242--1268, 2017{\natexlab{a}}.

\bibitem[Koren et~al.(2017{\natexlab{b}})Koren, Livni, and
  Mansour]{koren2017multi}
Tomer Koren, Roi Livni, and Yishay Mansour.
\newblock Multi-armed bandits with metric movement costs.
\newblock In \emph{Advances in Neural Information Processing Systems}, pages
  4119--4128, 2017{\natexlab{b}}.

\bibitem[Merhav et~al.(2002)Merhav, Ordentlich, Seroussi, and
  Weinberger]{merhav2002sequential}
Neri Merhav, Erik Ordentlich, Gadiel Seroussi, and Marcelo~J Weinberger.
\newblock On sequential strategies for loss functions with memory.
\newblock \emph{IEEE Transactions on Information Theory}, 48\penalty0
  (7):\penalty0 1947--1958, 2002.

\bibitem[Mohri and Yang(2014)]{mohri2014conditional}
Mehryar Mohri and Scott Yang.
\newblock Conditional swap regret and conditional correlated equilibrium.
\newblock In \emph{Advances in Neural Information Processing Systems}, pages
  1314--1322, 2014.

\bibitem[Mohri and Yang(2017)]{MohriYang2017}
Mehryar Mohri and Scott Yang.
\newblock Online learning with transductive regret.
\newblock In \emph{Advances in Neural Information Processing Systems}, pages
  5220--5230, 2017.

\bibitem[Neu et~al.(2010)Neu, Antos, Gy{\"o}rgy, and
  Szepesv{\'a}ri]{neu2010online}
Gergely Neu, Andras Antos, Andr{\'a}s Gy{\"o}rgy, and Csaba Szepesv{\'a}ri.
\newblock Online markov decision processes under bandit feedback.
\newblock In \emph{Advances in Neural Information Processing Systems}, pages
  1804--1812, 2010.

\bibitem[Roughgarden(2015)]{roughgarden2015intrinsic}
Tim Roughgarden.
\newblock Intrinsic robustness of the price of anarchy.
\newblock \emph{Journal of the ACM (JACM)}, 62\penalty0 (5):\penalty0 32, 2015.

\bibitem[Saha et~al.(2012)Saha, Jain, and Tewari]{saha2012interplay}
Ankan Saha, Prateek Jain, and Ambuj Tewari.
\newblock The interplay between stability and regret in online learning.
\newblock \emph{arXiv preprint arXiv:1211.6158}, 2012.

\bibitem[Stoltz and Lugosi(2007)]{stoltz2007learning}
Gilles Stoltz and G{\'a}bor Lugosi.
\newblock Learning correlated equilibria in games with compact sets of
  strategies.
\newblock \emph{Games and Economic Behavior}, 59\penalty0 (1):\penalty0
  187--208, 2007.

\bibitem[Sutton and Barto(1998)]{sutton1998reinforcement}
Richard~S Sutton and Andrew~G Barto.
\newblock \emph{Reinforcement learning: An introduction}, volume~1.
\newblock MIT press Cambridge, 1998.

\bibitem[Szepesv{\'a}ri(2010)]{szepesvari2010algorithms}
Csaba Szepesv{\'a}ri.
\newblock Algorithms for reinforcement learning.
\newblock \emph{Synthesis lectures on artificial intelligence and machine
  learning}, 4\penalty0 (1):\penalty0 1--103, 2010.

\bibitem[Yu et~al.(2009)Yu, Mannor, and Shimkin]{yu2009markov}
Jia~Yuan Yu, Shie Mannor, and Nahum Shimkin.
\newblock Markov decision processes with arbitrary reward processes.
\newblock \emph{Mathematics of Operations Research}, 34\penalty0 (3):\penalty0
  737--757, 2009.

\end{thebibliography}
\clearpage
\appendix
\rule{\textwidth}{4pt}
\begin{center}
\Large{\bf Supplementary Material to ``Policy Regret in Repeated Games"}
\end{center}
\rule{\textwidth}{1pt}
\vspace{3mm}
\section{Proofs of results from Section~\ref{sec:react_strat}}
\begin{theorem}
\label{thm:reg_incomp}
Let the space of actions be $\{0,1\}$ and define the $m$-memory bounded utility functions as follows:
\begin{equation*}
f_t(a_{t-m+1},..,a_t) = \begin{cases}
1 & a_{t-m+i} = a_{t-m+i+1} = 1\text{ for } i\in\{1,..,m-2\}\wedge a_{t-1} \neq a_{t}\\
\frac{1}{2} & a_{t-m+i} = a_{t-m+i+1} = 1\text{ for } i\in\{1,..,m-1\}\\
0 & \text{otherwise}
\end{cases}.
\end{equation*}
Assuming $m\geq 3$ is a fixed constant (independent of $T$) any sequence with with sublinear policy regret will have linear external regret and every sequence with sublinear external regret will have linear policy regret.
\end{theorem}
\begin{proof}[Proof of Theorem~\ref{thm_main:reg_incomp}]
Let $(a_t)_{t=1}^T$ be a sequence with sublinear policy regret. Then this sequence has utility at least $\frac{T}{2} - o(T)$ and so there are at most $o(T)$ terms in the sequence which are not equal to $1$. Let the subsequence consisting of all $a_t = 0$ be indexed by $\cI$. Define $\tilde\cI = \{t,t+1,\cdots,t+m-1 : t\in \cI\}$ and consider the subsequence of functions $(f_t)_{t \not\in \tilde\cI}$. This is precisely the sequence of functions which have utility $\frac{1}{2}$ with respect to the sequence of play $(a_t)_{t=1}^T$. Notice that the length of this sequence is at least $T - mo(T) = T - o(T)$. The utility of this sequence is $\sum_{t\not\in\tilde\cI}- f_t(a_{t-m+1},..,a_t) \sum_{t\not\in\tilde\cI} f_t(1,..,1) = \frac{T - o(T)}{2}$, however, this subsequence has linear external regret, since $\sum_{t\not\in\tilde\cI} f_t(a_{t-m+1},..,a_{t-1},0) = \sum_{t\not\in\tilde\cI} f_t(1,..,1,0) = T - o(T)$. Thus the external regret of $(a_t)_{t=1}^T$ is 
\begin{align*}
\sum_{t=1}^T [f_t(a_{t-m+1},..,0) - f_t(a_{t-m+1},..,a_t)] &= \sum_{t\not\in\tilde\cI} [f_t(a_{t-m+1},..,0) - f_t(a_{t-m+1},..,a_t)]\\
&+ \sum_{t\in\tilde\cI} [f_t(a_{t-m+1},..,0) - f_t(a_{t-m+1},..,a_t)]\\
&\geq \frac{T - o(T)}{2} + \sum_{t\in\tilde\cI} [f_t(a_{t-m+1},..,0) - f_t(a_{t-m+1},..,a_t)]\\
&\geq \frac{T}{2} - o(T),
\end{align*}
where the last inequality follows from the fact that the cardinality of $\tilde\cI$ is at most $o(T)$ and thus $\sum_{t\in\tilde\cI} [f_t(a_{t-m+1},..,0) - f_t(a_{t-m+1},..,a_t)] \geq -o(T)$.
\par
Assume that $(a_t)_{t=1}^T$ has sublinear external regret. From the above argument, it follows that the utility of the sequence is at most $o(T)$ (otherwise if the sequence has utility $\omega(T)$, we can repeat the previous argument and get a contradiction with the fact the sequence is no-external regret). This implies that the the policy regret of the sequence is $\sum_{t=1}^T [f_t(1,1,\cdots,1) - f_t(a_{t-m+1},..,a_t)] = \frac{T}{2} - o(T)$.
\end{proof}

\begin{theorem}
\label{thm:stable}
Let $(a_t)_{t=1}^T$ and $(b_t)_{t=1}^T$ be the action sequences of player $1$ and $2$ respectively and suppose that they are coming from no-regret algorithms with regrets $R_1(T)$ and $R_2(T)$ respectively. Assume that the algorithms are on average $(m,S(T))$ stable with respect to the $\ell_2$ norm. Then 
\begin{align*}
\expect{u_2(f_t(b_{0:t-m+1},b,\cdots,b),b) - \sum_{t=1}^T u_2(f_t(b_{0:t-1}),b_t)} \leq \|\P_2\|S(T) + R_2(T),
\end{align*}
for any fixed action $b\in\cA_2$, where $\P_2$ is the utility matrix of player 2. A similar inequality holds for any fixed action $a\in\cA_1$ and the utility of player $1$.
\end{theorem}
\begin{proof}[Proof of Theorem~\ref{thm_main:stable}]
\begin{align*}
&\expect{u_2(f_t(b_0,\cdots,b_{t-m+1},b,\cdots,b),b) - \sum_{t=1}^T u_2(f_t(b_0,\cdots,b_{t-2},b_{t-1}),b_t)}\\
&=\expect{u_2(f_t(b_0,\cdots,b_{t-m+1},b,\cdots,b),b) - \sum_{t=1}^T u_2(f_t(b_0,\cdots,b_{t-2},b_{t-1}),b)}\\
&+\expect{u_2(f_t(b_0,\cdots,b_{t-1}),b) - \sum_{t=1}^T u_2(f_t(b_0,\cdots,b_{t-2},b_{t-1}),b_t)}\\
&\leq\sum_{t=1}^T \|b^\top\P_2\|_2\|f_t(b_0,\cdots,b_{t-m+1},b,\cdots,b) - f_t(b_0,\cdots,b_{t-2},b_{t-1})\|_2 + R_2(T)\\
&\leq \|b\|_1\|\P_2\|S(T) + R_2(T) = \|\P_2\|S(T) + R_2(T)
\end{align*}
where the first inequality holds by Cauchy-Schwartz, the second inequality holds by using the $m$-stability of the algorithm, together with the inequality between $l_1$ and $l_2$ norms.
\end{proof}

The next theorems show that MWU and Exp3 are stable algorithms.
\begin{theorem}
\label{thm:eg_stab}
MWU is an on average $(m,m\sqrt{T})$ stable algorithm with respect to $\ell_1$, for any $m < o(\sqrt{T})$.
\end{theorem}
\begin{proof}[Proof of Theorem~\ref{thm:eg_stab}]
We think of MWU as Exponentiated Gradient (EG) where the loss vector has $i$-th entry equal to the negative utility if the player decided to play action $i$. Let the observed loss vector at time $j$ be $\hat l_j$ and the output distribution be $p_j$, then the update of EG can be written as $p_{j+1} = \text{arg}\min_{p \in \mathcal{C}} \langle \hat l_j, p \rangle + \frac{1}{\eta} D(p,p_j)$, where $C$ is the simplex of the set of possible actions and $D$ is the KL-divergence.  Using Lemma 3 in~\cite{saha2012interplay}, with $f(p) = \langle \hat l_j, p \rangle + \frac{1}{\eta} D(p,p_j)$ and the fact that the KL-divergence is 1-strongly convex over the simplex $\mathcal{C}$ with respect to the $\ell_1$ norm, we have:
\begin{align*}
\frac{1}{2\eta}\norm{p_j - p_{j+1}}_1^2 &\leq f(p_j) - f(p_{j+1}) = \langle p_j - p_{j+1}, \hat l_j\rangle - \frac{1}{\eta}D(p_{j+1},p_{j})\\
&\leq \norm{p_j-p_{j+1}}_1\norm{\hat l_j}_{\infty} \leq \norm{p_j-p_{j+1}}_1,
\end{align*}
where the second inequality follows from H{\"o}lder's inequality and the fact that $D(p_{j+1},p_j) \geq 0$. Thus with step size $\eta \sim \sqrt{\frac{1}{T}}$, we have $\norm{p_j - p_{j+1}}_1 \leq \frac{1}{2\sqrt{T}}$. Using triangle inequality, we can get $\norm{p_{j-1} - p_{j+1}}_1 \leq \norm{p_{j-1} - p_j}_1 + \norm{p_j - p_{j+1}}_1 \leq \frac{2}{2\sqrt{T}}$ and induction shows that $\norm{p_{j-m+1} - p_{j+1}}_1 \leq \frac{m}{2\sqrt{T}}_1$. Suppose for the last $m$ iterations, a fixed loss function $l_a$ was played instead and the resulting output of the algorithm becomes $\tilde p_{j+1}$. Then using the same argument as above we have $\norm{p_{j-m+1} - \tilde p_{j+1}}_1 \leq \frac{m}{2\sqrt{T}}$ and thus $\norm{p_{j+1} - \tilde p_{j+1}}_1 \leq \frac{m}{\sqrt{T}}$. Summing over all $T$ rounds concludes the proof.
\end{proof}

\begin{theorem}
\label{thm:exp3_stab}
Exp3 is an on average $(m,m\sqrt{T})$ stable algorithm with respect to $\ell_1$, for any $m < o(\sqrt{T})$
\end{theorem}
\begin{proof}[Proof of Theorem~\ref{thm:exp3_stab}]
The update at time $t$, conditioning on the the draw being $i$, is given by $p_{t+1}^i = \frac{w_t^i\exp{\frac{\gamma}{kp_{t}^i} u_t^i}}{w_t^i\exp{\frac{\gamma}{kp_t^i} u_t^i}+ \sum_{j\neq i}w_t^j}$ and for $j\neq i$, $p_{t+1}^j = \frac{w_t^j}{w_t^i\exp{\frac{\gamma}{kp_t^i} u_t^i}+ \sum_{j\neq i}w_t^j}$, where $u_t$ is the utility vector at time $t$, $w_t$ is the weight vector at time $t$ i.e. $w_{t+1}^i = w_t^i \exp{\frac{\gamma}{kp_t^i} u_t^i}, w_{t+1}^j = w_t^j$, and $k$ is the number of actions. We have the following bound:
\begin{align*}
|p_{t+1}^i - p_t^i| &= \left|\frac{w_t^i\exp{\frac{\gamma}{kp_t^i} u_t^i}}{w_t^i\exp{\frac{\gamma}{kp_t^i} u_t^i}+ \sum_{j\neq i}w_t^j} - \frac{w_t^i}{\sum_{j}w_t^j}\right|\leq \left|\frac{w_t^i(\exp{\frac{\gamma}{kp_t^i} u_t^i}-1)}{\sum_{j}w_t^j}\right|\\
&= p_t^i(\exp{\frac{\gamma}{kp_t^i} u_t^i}-1) \leq p_t^i 2\frac{\gamma}{kp_t^i}u_t^i \leq 2\frac{\gamma}{k},
\end{align*}
where the first inequality uses the fact that $p_{t+1}^i \geq p_t^i$ and the second inequality uses the choice of $\gamma$ together with $\exp{x} \leq 2x+1$ for $x\in [0,1]$. Similarly for $j\neq i$, we have:
\begin{align*}
|p_{t+1}^j - p_{t}^j| &= \left|\frac{w_t^j}{w_t^i\exp{\frac{\gamma}{kp_t^i} u_t^i}+ \sum_{j\neq i}w_t^j} - \frac{w_t^j}{\sum_{j}w_t^j}\right| \leq \left| \frac{w_t^j}{\sum_{j}w_t^j}\left(1 - \frac{1}{\exp{\frac{\gamma u_t^i}{kp_t^i}}}\right)\right|\\
&= p_t^j\left(1 - \exp{-\frac{\gamma u_t^i}{kp_t^i}}\right) \leq p_t^j\frac{\gamma u_t^i}{kp_t^i} \leq \frac{p_t^j}{k p_t^i}\gamma,
\end{align*}
where we have used $p_{t+1}^j \leq p_t^j$ and $\exp{-x} \geq 1-x$ for all $x$. We can now proceed to bound $\expectation{i_t}{\|p_{t+1} - p_t\|_1|i_{1:t-1}}$, where $i_t$ is the random variable denoting the draw at time $t$:
\begin{align*}
\expectation{i_t}{\|p_{t+1} - p_t\|_1|i_{1:t-1}} &= \sum_{i} p_t^i\|p_{t+1} - p_t\|_1 = \sum_{i} p_t^i\left(|p_{t+1}^i - p_t^i| + \sum_{j\neq i} |p_{t+1}^j - p_t^j|\right)\\
&\leq 2\gamma + \sum_{i,j}p_t^i\frac{p_t^j}{k p_t^i}\gamma = 3\gamma.
\end{align*}
Setting $\gamma \sim \frac{1}{\sqrt{T}}$ finishes the proof.
\end{proof}

Combining the above results together, we can show that both players will also have no-policy regret for any $m<o(\sqrt{T})$.
\begin{corollary}
\label{cor:eg_npr}
Let $(a_t)_{t=1}^T$ and $(b_t)_{t=1}^T$, be the action sequences of players $1$ and $2$ respectively and suppose that the sequences are coming from MWU. Then for any fixed $m$, it holds:
\begin{align*}
\expect{u_2(f_t(b_{0:t-m+1},b,\cdots,b),b) - \sum_{t=1}^T u_2(f_t(b_{0:t-1}),b_t)} \leq O(\|\P_2\|m\sqrt{T}),
\end{align*}
for any fixed action $b\in \mathcal{A}_2$, where $f_t$ are the functions corresponding to the MWU algorithm used by player 1.
\end{corollary}
\begin{proof}
From Theorem~\ref{thm:eg_stab} it follows that MWU is on average $(m,m\sqrt{T})$ stable and has regret at most $O(\sqrt{T})$.
\end{proof}

\section{Proofs of results from Section~\ref{sec:pol_eq}}
\begin{theorem}
\label{thm:pol_reg_cons}
Let $\pi$ be a distribution over the product of function spaces $\cF_1\times\cF_2$. There exists a stationary distribution $\tilde\sigma$ of the Markov chain $\M_a$ for any fixed $a\in\cA_1$ and $(f,g)\sim\pi$ such that $\expectation{(a,b)\sim\tilde\sigma}{u_1(a,b)} = \expectation{(f,g)\sim\pi}{u_1(a,g(a))}$
\end{theorem}
\begin{proof}[Proof of Theorem~\ref{thm_main:pol_reg_cons}]
Note that, by definition $(\M_a)_{(\tilde a,\tilde b),(\hat a,\hat b)} = 0$ if $\hat a\neq a$ and $M_{(\tilde a,\tilde b),(\hat a,\hat b)} = \sum_{f,g: g(\tilde a) = \hat b} \pi(f,g)$ if $\hat a = a$. Consider the distribution $\tilde\sigma$ over $\cA$, where $\tilde\sigma_{(\tilde a,\tilde b)} = 0$ if $\tilde a \neq a$ and $\tilde\sigma_{(\tilde a,\tilde b)} = \sum_{f,g: g(a) = \tilde b} \pi(f,g)$ if $\tilde a = a$. We now show that $\tilde \sigma$ is a stationary distribution of $\M_a$:
\begin{align*}
\left(\tilde\sigma^\top \M_a\right)_{(a,\tilde b)} &= \sum_{(\hat a,\hat b)} \tilde\sigma_{(\hat a,\hat b)} (\M_a)_{(\hat a,\hat b), (a,\tilde b)} = \sum_{\hat b} \tilde\sigma_{(a,\hat b)}(\M_a)_{(a,\hat b), (a,\tilde b)}\\
&= \sum_{\hat b}\left(\sum_{f,g: g(a) = \hat b}\pi(f,g)\right)\left(\sum_{f,g: g(a) = \tilde b} \pi(f,g)\right)\\
&=\left(\sum_{f,g: g(a) = \tilde b} \pi(f,g)\right)\left(\sum_{\hat b}\sum_{f,g: g(a) = \hat b}\pi(f,g)\right)\\
&=\sum_{f,g: g(a) = \tilde b} \pi(f,g) = \tilde\sigma_{(a,\tilde b)}.
\end{align*}
Finally, notice that:
\begin{align*}
\expectation{(f,g)\sim\pi}{u_1(a,g(a))} = \sum_{b\in\mathcal{A}_2} u_1(a,b)\prob{g(a)=b} = \sum_{b\in\mathcal{A}_2} u_1(a,b) \sum_{f,g: g(a)=b}\pi(f,g).
\end{align*}
\end{proof}

\subsection{Sketch of proof of the main result}
The proof of Theorem~\ref{thm_main:main_result} has three main steps. The first step defines the natural empirical Markov chains $\hat\M$, $\hat\M_a$ and $\hat\M_b$ from the empirical play $(p_t)_{t=1}^t$ and shows that the empirical distributions $\hat\sigma$, $\hat\sigma_a$ and $\hat\sigma_b$ are stationary distributions of the respective Markov chains. The next step is to show that the empirical Markov chains converge to Markov chains $\M$, $\M_a$ and $\M_b$ induced by some distribution $\pi$ over $\cF$. The final step is to show that $\pi$ is a policy equilibrium.

We begin with the definition of the empirical Markov chains.
\begin{definition}
\label{def:emp_proc}
Let the empirical Markov chain be 
$\hat\M$, with $\hat\M_{i,j} = \frac{\frac{1}{T}\sum_{t=1}^T p_{t}(x_i)p_t(x_j)}{\frac{1}{T}\sum_{t=1}^T p_t(x_i)}$ if $\frac{1}{T}\sum_{t=1}^T p_t(x_i) \neq 0$ and $0$ otherwise, where $p_t$ is defined in~\ref{def:emp_distr_dev}. For any fixed $a\in\cA_1$, let the empirical Markov chain corresponding to the deviation in play of player $1$ be $\hat\M_a$, with $(\hat\M_a)_{i,j} = \frac{\frac{1}{T}\sum_{t=1}^T (p_a)_{t}(x_i)(p_a)_t(x_j)}{\frac{1}{T}\sum_{t=1}^T (p_a)_t(x_i)}$, if $\frac{1}{T}\sum_{t=1}^T (p_a)_t(x_i) \neq 0$ and $0$ otherwise, where $(p_a)_t$ is defined in~\ref{def:emp_distr_dev}. The Markov chain $\hat\M_b$ is defined similarly for any $b\in\cA_2$.
\end{definition}
The intuition behind constructing these Markov chains is as follows -- if we were only provided with the observed empirical play $(x_t)_{t=1}^T=(a_t,b_t)_{t=1}^T$ and someone told us that the $x_t$'s were coming from a Markov chain, we could try to build an estimator of the Markov chain by approximating each of the transition probabilities. In particular the estimator of transition from state $i$ to state $j$ is given by $\tilde\M_{i,j} = \frac{\frac{1}{T}\sum_{t=1}^T\delta_{t-1}(x_i)\delta_{t}(x_j)}{\frac{1}{T}\sum_{t=1}^T\delta_t(x_i)}$, where $\delta_t(x_i) = 1$ if $x_i$ occurred at time $t$ and $0$ otherwise. When the players are playing according to a no-regret algorithm i.e.\ at time $t$, $x_t$ is sampled from $p_t$, it is possible to show that $\tilde\M_{i,j}$ concentrates to $\hat\M_{i,j}$ (see section~\ref{sec:conc_of_matrices}). Not only does $\hat\M$ arise naturally, but it turns out that the empirical distribution $\hat\sigma$ defined in~\ref{def:emp_distr_dev} is also a stationary distribution of $\hat\M$. 

\begin{lemma}
\label{lem:stationary}
The distribution of play $\hat\sigma = \frac{1}{T} \sum_{t=1}^T p_t$ is a stationary distribution of $\hat\M$. Similarly the distributions $\tilde\sigma,\hat\sigma_a,\tilde\sigma_a,\hat\sigma_b$ and $\tilde\sigma_b$ are stationary distributions of the Markov chains $\tilde\M,\hat\M_a,\tilde\M_a,\hat\M_b$ and $\tilde\M_b$ respectively.
\end{lemma}
\begin{proof}
We show the result for $\hat\sigma$ and $\hat\M$. The rest of the results can then be derived in the same way.
\begin{align*}
(\hat\sigma^\top \hat\M)_j = \sum_{i=1}^{|\cA|}\left(\frac{1}{T}\sum_{t=1}^T p_t(x_i)\right) \frac{\frac{1}{T}\sum_{t=1}^T p_{t}(x_i)p_t(x_j)}{\frac{1}{T}\sum_{t=1}^T p_t(x_i)} = \frac{1}{T}\sum_{t=1}^T p_t(x_j)\sum_{i=1}^{|\cA|}p_{t}(x_i) = \frac{1}{T}\sum_{t=1}^T p_t(x_j),
\end{align*}
where the first equality holds because the $i$-th entry of the vector $\hat\sigma$ is exactly $\frac{1}{T}\sum_{t=1}^T p_t(x_i)$ and the $(i,j)$-th entry of $\hat\M$ by definition is $\frac{\frac{1}{T}\sum_{t=1}^T p_{t}(x_i)p_t(x_j)}{\frac{1}{T}\sum_{t=1}^T p_t(x_i)}$, and the last equality holds because $\p_t$ is a distribution over actions so $\sum_{i=1}^{|\cA|}p_{t}(x_i) = 1$.
\end{proof}

Suppose that both players are playing MWU for $T$ rounds.  Then Lemma~\ref{lem:stationary} together with Theorems~\ref{thm_main:stable} and the stability of MWU imply that $\expectation{(a,b)\sim\hat\sigma}{u_1(a,b)} \geq \expectation{(a,b)\sim\hat\sigma_a}{u_1(a,b)} - O(m/\sqrt{T})$. A similar inequality holds for player $2$ and $\hat\sigma_b$. As $T\rightarrow\infty$, the inequality above becomes similar to~\eqref{eq:yes}. This will play a crucial role in the proof of our convergence result, which shows that $\hat\sigma,\hat\sigma_a$ and $\hat\sigma_b$ converge to the set of policy equilibria. We would also like to guarantee that the empirical distributions of observed play $\tilde\sigma,\tilde\sigma_a$ and $\tilde\sigma_b$ also converge to this set. To show this second result, we are going to proof that $\tilde\sigma$ approaches $\hat\sigma$ almost surely as $T$ goes to infinity.

\begin{lemma}
\label{lem_main:conc}
Let $\hat\sigma = \frac{1}{T}\sum_{t=1}^T p_t$ be the empirical distribution after $T$ rounds of the game and let $\tilde\sigma = \frac{1}{T}\sum_{t=1}^T \delta_t$ be the empirical distribution of observed play. Then $\lim\sup_{T\rightarrow\infty}\norm{\tilde\sigma - \hat\sigma}_1 = 0$ almost surely. Similarly, for the distributions corresponding to deviation in play we have $\lim\sup_{T\rightarrow\infty}\norm{\tilde\sigma_a - \hat\sigma_a}_1 = 0$ and $\lim\sup_{T\rightarrow\infty}\norm{\tilde\sigma_b - \hat\sigma_b}_1 = 0$ almost surely.
\end{lemma}
\begin{proof}
The proof follows the same reasoning as the proof of lemma~\ref{thm:conc}.
\end{proof}

Our next step is to show that the empirical Markov chains $\hat\M$ converge to a Markov chain $\M$ induced by some distribution $\pi$ over the functional space $\cF$. We do so by constructing a sequence of empirical distributions $\hat\pi$ over $\cF$, based on the players' strategies, which induce $\hat\M$. We can then consider every convergent subsequence of $(\hat\pi)_{T_\ell=1}^\infty$ with limit point $\pi$ and argue that the corresponding sequence $(\hat\M)_{T_\ell=1}^\infty$ of Markov chains converges to the Markov chain induced by $\pi$.
\begin{definition}
\label{def:emp_func_distr}
Let $\hat\pi$ be the distribution over $\cF$, such that the probability to sample any fixed $f:\cA_2\rightarrow\cA_1$ and $g:\cA_1\rightarrow\cA_2$ is $\hat\pi(f,g) = \prod_{i\in|\cA|}\frac{\sum_{t} p_{t}(x_i)p_t(y_i)}{\sum_{t=1} p_t(x_i)}$, where $x_i = (a_i,b_i)$ and $y_i = (f(b_i),g(a_i))$. Similarly, let $\hat\pi_a$ and $\hat\pi_b$ be the distributions over $\cF$ constructed as above but by using the empirical distribution of deviated play induced by player $1$ deviating to action $a\in\cA_1$ and player $2$ deviating to action $b\in\cA_2$.
\end{definition}
The next lemma checks that $\hat\pi$ is really a probability distribution.
\begin{lemma}
\label{lem:distr_check}
The functionals $\hat\pi,\hat\pi_a$ and $\hat\pi_b$ are all probability distributions.
\end{lemma}
\begin{proof}
Consider the space of all transition events for a fixed $(a,b)$ pair i.e. $\mathcal{S}_{(a,b)} = \{\left((a',b')\times(a,b)\right) : (a',b') \in \cA\}$. There is an inherent probability measure on this set, given by $\prob{(a',b')\times(a,b)} = \frac{\sum_{t}p_{t}(a,b)p_t(a',b')}{\sum_{t} p_t(a,b)}$. It is easy to see that this is a probability measure, since the measure of the whole set is exactly 
\begin{align*}
\sum_{(a',b')\in\cA}\frac{\sum_{t}p_{t}(a,b)p_t(a',b')}{\sum_{t} p_t(a,b)} = \frac{\sum_{t}p_{t}(a,b)\sum_{(a',b')\in\cA}p_t(a',b')}{\sum_{t} p_t(a,b)} = 1.
\end{align*}
The set of all $\cF$ can exactly be thought of as $\times_{(a,b)\in\cA} \mathcal{S}_{(a,b)}$ and the function $\hat\pi$ defined in~\ref{def:emp_func_distr} is precisely the product measure on that set. Similar arguments show that $\hat\pi_a$ and $\hat\pi_b$ are probability distributions.
\end{proof}

The proof of the above lemma reveals something interesting about the construction of $\hat\pi$. Fix the actions $(a,b)\in\cA$. Then the probability to sample a function pair $(f,g)$ which map $(a,b)$ to $(a',b')$ i.e.\ $a' = f(a)$ and $b'= f(b)$ is precisely equal to the entry $\hat\M_{(a,b),(a',b')}$ of the empirical Markov chain. Since every function pair $(f,g)\in\cF$ is determined by the way $\cA$ is mapped, and we have already have a probability distribution for  a fixed mapping $(a,b)$ to $(a',b')$, we can just extend this to $\hat\pi$ by taking the product distribution over all pairs $(a,b)\in\cA$. This construction gives us exactly that $\hat\M$ is induced by $\hat\pi$.
\begin{lemma}
\label{lem:emp_func_distr}
Let $\hat\M$, $\hat\M_a$ and $\hat\M_b$ be the empirical Markov chains defined in~\ref{def:emp_proc}, then the induced Markov chain from $\hat\pi$ is exactly $\hat\M$ and the induced Markov chains from $\hat\pi_a$ and $\hat\pi_b$ are exactly $\hat\M_a$ and $\hat\M_b$.
\end{lemma}
\begin{proof}
Consider $\hat\M_{(a,b),(a',b')} = \frac{\sum_{t=1}^T p_t(a',b')p_{t}(a,b)}{\sum_{t=1}^T p_t{(a,b)}}$. The transition probability induced by $\hat\pi$ is exactly 
\begin{align*}
&\prob{(a',b')|(a,b)} = \sum_{(f,g): (f(b), g(a)) = (a',b')}\hat\pi(f,g) = \sum_{(f,g): (f(b), g(a)) = (a',b')}\prod_{i\in[|\cA|]}\frac{\sum_{t} p_{t}(x_i)p_t(y_i)}{\sum_{t=1} p_t(x_i)}\\
&=\sum_{(f,g): (f(b), g(a)) = (a',b')}\frac{\sum_{t} p_{t}(a,b)p_t(a',b')}{\sum_{t=1} p_t(a,b)} \prod_{i\in[|\cA|],(x_i,y_i)\neq((a,b),(a',b')}\frac{\sum_{t} p_{t}(x_i)p_t(y_i)}{\sum_{t=1} p_t(x_i)}\\
&=\frac{\sum_{t} p_{t}(a,b)p_t(a',b')}{\sum_{t=1} p_t(a,b)}\sum_{(f,g): (f(b), g(a)) = (a',b')}\prod_{i\in[|\cA|],(x_i,y_i)\neq((a,b),(a',b')}\frac{\sum_{t} p_{t}(x_i)p_t(y_i)}{\sum_{t=1} p_t(x_i)}\\
&=\frac{\sum_{t} p_{t}(a,b)p_t(a',b')}{\sum_{t=1} p_t(a,b)},
\end{align*}
where the last equality holds, because for fixed $(f,g)$ with $x_i=(a_i,b_i)$ and $y_i = (f(b_i),g(a_i))$, the product $\prod_{i\in[|\cA|],(x_i,y_i)\neq((a,b),(a',b')}\frac{\sum_{t} p_{t}(x_i)p_t(y_i)}{\sum_{t=1} p_t(x_i)}$ is exactly the conditional probability $\hat\pi((f,g)|(f(b),g(a)) = (a',b'))$. The result for $\hat\pi_a$ and $\hat\pi_b$ is shown similarly.
\end{proof}

The last step of the proof is to show that any limit point $\pi$ of $(\hat\pi)_T$ is necessarily a policy equilibrium. This is done through an argument by contradiction. In particular we assume that a limit point $\pi$ is not a policy equilibrium. The limit point $\pi$ induces a Markov chain $\M$, which we can show is the limit point of the corresponding subsequence of $(\hat\M)_T$ by using lemma~\ref{lem:emp_func_distr}. Since $\pi$ is not a policy equilibrium, no stationary distribution of $\M$ can satisfy the inequalities~\eqref{eq:yes}. We can now show that the subsequence of $(\hat\sigma)_T$ which are stationary distributions of the corresponding $\hat\M$'s, converges to a stationary distribution of $\M$. This, however, is a contradiction because of the next theorem. 
\begin{theorem}
\label{thm:conv_thm}
Let $P$ be the set of all product distributions $\sigma\times\sigma_a\times\sigma_b$ which satisfy the inequalities in~\ref{eq:yes}:
\begin{align*}
\expectation{(a,b)\sim\sigma}{u_1(a,b)} &\geq \expectation{(a,b)\sim\sigma_a}{u_1(a,b)}\\
\expectation{(a,b)\sim\sigma}{u_2(a,b)} &\geq \expectation{(a,b)\sim\sigma_b}{u_2(a,b)}.
\end{align*}
Let $\hat\sigma^T$ be the empirical distribution of play after $T$ rounds and let $\hat\sigma^T_a$ be the empirical distribution when player $1$ switches to playing action $a$ and define $\hat\sigma^T_b$ similarly for player $2$. Then the product distribution $\hat\sigma^T\times\hat\sigma_a^T\times\hat\sigma_b^T$ converges to weakly to the set $P$.
\end{theorem}
\begin{proof}
Theorem~\ref{thm:conv_thm} follows from the fact that convergence in the Prokhorov metric implies weak convergence. First notice that by Prokhorov's theorem $\cP(\cA)$ is a compact metric space with the Prokhorov metric. Thus by Tychonoff's Theorem the product space $\cP(\cA)^3$ is compact in the maximum metric. Suppose for a contradiction that the sequence $(\hat\sigma^T\times\hat\sigma_a^T\times\hat\sigma_b^T)_T$ does not converge to the set $S$. This implies that there exists some subsequence $(\hat\sigma^k\times\hat\sigma_a^k\times\hat\sigma_b^k)_k$, converging to some $\hat\sigma\times\hat\sigma_a\times\hat\sigma_b \not\in S$. If $\hat\sigma\times\hat\sigma_a\times\hat\sigma_b \not\in S$, then either $\expectation{(a,b)\sim\sigma}{u_1(a,b)} < \expectation{(a,b)\sim\sigma_a}{u_1(a,b)}$ or $\expectation{(a,b)\sim\sigma}{u_2(a,b)} < \expectation{(a,b)\sim\sigma_b}{u_2(a,b)}$. WLOG suppose the first inequality holds. From our assumption, the continuity of $u_1$ and the definition of the maximum metric we have $\lim_{k\rightarrow\infty}\expectation{(a,b)\sim \hat\sigma^k}{u_1(a,b)} = \expectation{(a,b)\sim\hat\sigma}{u_1(a,b)}$ and $\lim_{k\rightarrow\infty}\expectation{(a,b)\sim \hat\sigma^k_a}{u_1(a,b)} = \expectation{(a,b)\sim\hat\sigma_a}{u_1(a,b)}$. Notice that by the fact $\hat\sigma_a^k$ is the average empirical distribution if player $1$ changed its play to the fixed action $a\in\cA_1$ and $\hat\sigma^k$ being the average empirical distribution it holds that $\expectation{(a,b)\sim\hat\sigma^k}{u_1(a,b)} - \expectation{(a,b)\sim\hat\sigma_a^k}{u_1(a,b)} \geq -o(1)$ and thus $\lim_{k\rightarrow\infty} \left[\expectation{(a,b)\sim\hat\sigma^k}{u_1(a,b)} - \expectation{(a,b)\sim\hat\sigma_a^k}{u_1(a,b)}\right] \geq 0$. The above implies:
\begin{align*}
0 &\leq \lim_{k\rightarrow\infty} \left[\expectation{(a,b)\sim\hat\sigma^k}{u_1(a,b)} - \expectation{(a,b)\sim\hat\sigma_a^k}{u_1(a,b)}\right]\\
&= \lim_{k\rightarrow\infty} \expectation{(a,b)\sim\hat\sigma^k}{u_1(a,b)} - \lim_{k\rightarrow\infty} \expectation{(a,b)\sim\hat\sigma_a^k}{u_1(a,b)}\\
&=\expectation{(a,b)\sim\hat\sigma}{u_1(a,b)} - \expectation{(a,b)\sim\hat\sigma_a}{u_1(a,b)} < 0,
\end{align*}
which is a contradiction. Since $\cA\times\cA\times\cA$ is separable then convergence in the Prokhorov metric in $\cP(\cA\times\cA\times\cA)$ is equivalent to weak convergence. Again we can argue by contradiction -- if we assume that $(\hat\sigma^T\times\hat\sigma_a^T\times\hat\sigma_b^T)_T$ doesn't converge to the set $S$ in the Prokhorov metric, then there exists some subsequence $(\hat\sigma^k\times\hat\sigma_a^k\times\hat\sigma_b^k)_k$ which converges to some $\mu \in \cP(\cA^3)$ such that $\mu \not\in S$. First we argue that $\mu$ must be a product measure i.e. $\mu = (\hat\sigma\times\hat\sigma_a\times\hat\sigma_b)$. Let $(\hat\sigma^j\times\hat\sigma_a^j\times\hat\sigma_b^j)_j$ be a convergence subsequence of $(\hat\sigma^k\times\hat\sigma_a^k\times\hat\sigma_b^k)_k$ in $\cP(\cA)^3$, with limit $(\hat\sigma\times\hat\sigma_a\times\hat\sigma_b)$, then each of $\hat\sigma^j$, $\hat\sigma_a^j$ and $\hat\sigma_b^j$ converge weakly to $\hat\sigma$, $\hat\sigma_a$ and $\hat\sigma_b$ respectively and thus $(\hat\sigma^j\times\hat\sigma_a^j\times\hat\sigma_b^j)_j$ converges weakly to $\hat\sigma\times\hat\sigma_a\times\hat\sigma_b$ and thus it converges in the Prokhorov metric of $\cP(\cA^3)$. This implies that $(\hat\sigma^k\times\hat\sigma_a^k\times\hat\sigma_b^k)_k$ also converges weakly to $\hat\sigma\times\hat\sigma_a\times\hat\sigma_b$ and so $\mu$ is a product measure. Again since $\mu \not\in S$, assume WLOG $\expectation{(a,b)\sim\hat\sigma}{u_1(a,b)} < \expectation{(a,b)\sim\hat\sigma_a}{u_1(a,b)}$. Define $f:\cA^3\rightarrow \mathbb{R}$, $f(a,b,c,d,e,f) = u_1(a,b) - u_1(c,d)$. $f$ is continuous and from the no-policy regret of the pair $\hat\sigma^k,\hat\sigma^k_a$ we have:
\begin{align*}
0 \leq \lim_{k\rightarrow\infty}\expectation{(a,b,c,d,e,f)\sim(\hat\sigma^k\times\hat\sigma_a^k\times\hat\sigma_b^k)_k}{f(a,b,c,d,e,f)} = \expectation{(a,b,c,d,e,f)\sim\mu}{f(a,b,c,d,e,f)} < 0,
\end{align*}
which is again a contradiction.
\end{proof}
For completeness we restate the main result below and give its proof in full.
\begin{theorem}
\label{thm:main_result}
If the algorithms played by player $1$ in the form of $f_t$ and player $2$ in the form of $g_t$ give sub-linear policy regret sequences, then sequence of product distributions $(\hat\sigma^T\times\hat\sigma_a^T\times\hat\sigma_b^T)_{T=1}^\infty$ converges weakly to the set $S$.
\end{theorem}
\begin{proof}[Proof of Theorem~\ref{thm_main:main_result}]
We consider the sequence of empirical distributions $\hat\pi^T$ defined in~\ref{def:emp_func_distr}, over the functional space $\cF_1\times\cF_2$ and show that this sequence must converge to the set of all policy equilibria $\Pi$ in the Prokhorov metric on $\cP(\cF_1\times\cF_2)$. First, notice that since the functions $f : \cA_2 \rightarrow \cA_1$ are from finite sets of actions to finite sets of actions, we can consider the set $\cF_1$ as a subset of a finite dimensional vector space, with the underlying field of real numbers and the metric induced by the $l_1$ norm. Similarly, we can also equip $\cF_2$ with the $l_1$ norm. Since both $\cF_1$ and $\cF_2$ are closed sets with respect to this metric and they are clearly bounded, they are compact. Thus the set $\cF_1\times\cF_2$ is a compact set with the underlying metric $d$ being the maximum metric. By Prokhorov's theorem we know that $\cP(\cF_1\times\cF_2)$ is a compact metric space with the Prokhorov metric. Suppose that the sequence $(\hat\pi^T)_T$ does not converge to $\Pi$. This implies that there is some convergent subsequence $(\hat\pi^t)_t$ with a limit $\pi$ outside of $\Pi$. Let $\M$ be the Markov chain induced by $\pi$ and let $\hat\M^T$ be the Markov chain induced by $\hat\pi^T$. 

First we show that $\lim_{t\rightarrow\infty} \|\hat\M^t - \M\|_1 = 0$. Recall that $\M_{(a,b),(a',b')} = \sum_{(f,g):f(b) = a',g(a) = b'} \pi(f,g)$ and that by lemma~\ref{lem:emp_func_distr} $\hat\M^t_{(a,b),(a',b')} = \sum_{(f,g):f(b) = a',g(a) = b'} \hat\pi^t(f,g)$. Notice that $f,g$ are continuous functions on $\cF_1$ and $\cF_2$, since the topology induced by the $l_1$ metric on both sets is exactly the the discrete topology and every function from a topological space equipped with the discrete topology is continuous. Since convergence in the Prokhorov metric implies weak convergence, we have that for any fixed $f,g$, $\lim_{t\rightarrow\infty} \hat\pi^t(f,g) = \pi(f,g)$. Since the sum $\sum_{(f,g):f(b) = a',g(a) = b'} \hat\pi^t(f,g)$ is finite this implies that $\lim_{t\rightarrow\infty}\sum_{(f,g):f(b) = a',g(a) = b'} \hat\pi^t(f,g) = \sum_{(f,g):f(b) = a',g(a) = b'} \pi(f,g)$ and so $\lim_{t\rightarrow\infty} \|\hat\M^t - \M\|_1 = 0$. 

Next we show that any convergent subsequence $(\hat\sigma^k)_k$ of $(\hat\sigma^t)_t$ in the Prokhorov metric, converges to a stationary distribution $\sigma$ of $\M$. First notice that $(\hat\sigma^k)_k$ exists, since $\cP(\cA)$ is compact. Next, suppose $\sigma$ is the limit of $(\hat\sigma^k)_k$ in the Prokhorov metric. This implies that $\lim_{k\rightarrow \infty} \hat\sigma^k(a,b) = \sigma(a,b)$, in particular if we consider $\cA \subset \mathbb{R}^{|\cA|}$ and $\hat\sigma^k,\sigma \in \mathbb{R}^{|\cA|}$ as vectors, then the above implies that $\lim_{k\rightarrow\infty}\|\sigma - \hat\sigma^k\|_1 = 0$.  Next we construct the following sequence $(\sigma_{k_n})_{k_n}$ of stationary distributions of $\M$ -- choose $k_n$ large enough, so that $\|\M - \hat\M^{k_n}\| \leq \frac{1}{n}$. Such a $k_n$ exists, because $(\hat\M^k)_k$ is a subsequence of $(\hat\M^t)_t$ which converges to $\M$. By lemma~\ref{lem:stationary}, there exists a stationary distribution $\sigma_{k_n}$ of $\M$ such that $\|\hat\sigma_{k_n} - \sigma_{k_n}\|_1 \leq \frac{c}{n}$, for some constant $c$. We show that $\sigma_{k_n}$ converges to $\sigma$. Fix some $\epsilon > 0$, we find an $N$, such that for any $n\geq N$ we have $\|\sigma_{k_n} - \sigma\|_1 < \epsilon$. Notice that $\|\sigma_{k_n} - \sigma\|_1 \leq \|\sigma_{k_n} - \hat\sigma_{k_n}\|_1 + \|\hat\sigma_{k_n} - \sigma\|_1$. Since $\|\sigma_{k_n} - \hat\sigma_{k_n}\|_1 \leq \frac{c}{n}$ and by convergence, we know that for $\frac{\epsilon}{2}$, there exists $N'$ such that for any $n\geq N'$, $\|\hat\sigma_{k_n} - \sigma\|_1 < \frac{\epsilon}{2}$, we can set $N = \max\left(\frac{2}{c\epsilon}, N_1\right)$. Suppose, for a contradiction, that $\sigma$ is not a stationary distribution of $\M$. Then there exists some $\epsilon$ such that $\|\sigma^\top\M - \sigma\|_2 > \epsilon$. This implies:
\begin{align*}
\epsilon < \|\sigma^\top\M - \sigma \|_2 \leq \|\sigma^\top\M - \sigma_{k_n}^\top\M\|_2 + \|\sigma_{k_n}^\top\M - \sigma\|_2 < 2\|\sigma_{k_n} - \sigma\|_2,
\end{align*}
where the last inequality holds by the fact $\sigma-\sigma_{k_n}$ is not a stationary distribution of $\M$ and thus $\M$ can only shrink the difference as a stochastic matrix. The inequality $2\|\sigma_{k_n} - \sigma\|_2 > \epsilon$ is a contradiction since we know that $\sigma_{k_n}$ converges to $\sigma$ and thus $\sigma$ is a stationary distribution of $\M$. Since strong convergence, implies weak convergence, which in hand implies convergence in the Prokhorov metric for separable metric spaces, we have shown that every convergent subsequence of $(\hat\sigma_t)_t$ converges to a stationary distribution of $\M$ in the Prokhorov metric.

Next, we show that $(\hat\pi_a^t)_t$ converges to $\pi_a$. By assumption $(\hat\pi^t)_t$ converges weakly to $\pi$. Since we are are in a finite dimensional space, we also have strong convergence. In particular, for any $g\in\cF_2$, we have 
\begin{align*}
\lim_{t\rightarrow\infty}\sum_{f\in\cF_1} \hat\pi^t(f,g) = \sum_{f\in\cF_1}\lim_{t\rightarrow\infty}\hat\pi^t(f,g) = \sum_{f\in\cF_1} \pi(f,g)
\end{align*}
and so the sequence of marginal distribution also converges to the respective marginal of $\pi$. Since $\hat\pi_a^t$ is exactly the product distribution of the dirac distribution over $\cF_1$ putting all weight on the constant function mapping everything to the fixed action $a$ and the marginal of $\hat\pi^t$ over $\cF_2$, by the convergence of marginals we conclude that $(\hat\pi^t_a)_t$ converges to $\pi_a$ in the strong sense and thus in the Prokhorov metric. In the same way we can show that $(\hat\pi^t_b)_t$ converges to $\pi_b$.

With a similar argument as for $(\hat\sigma^t)$ we show that every convergent subsequence of $(\hat\sigma_a^t)_t$ converges to a stationary distribution of $\M_a$ and any convergent subsequence of $(\hat\sigma_b^t)_t$ converges to a stationary distribution of $\M_b$. Because of the construction in theorem~\ref{thm:pol_reg_cons} and the convergence of $\hat\pi_a$ to $\pi_a$, we can guarantee that $(\hat\sigma_a^t)_t$ converges precisely to $\sigma_a$:
\begin{align*}
\sigma_a(a,b) &= \sum_{f,g: g(a) =  b} \pi_a(f,g) = \sum_{f,g: g(a) = b} \lim_{t\to\infty} \hat\pi_a^t(f,g)\\
&= \lim_{t\to\infty} \sum_{f,g: g(a) = b}  \hat\pi_a^t(f,g) = \lim_{t\to\infty}\hat\sigma_a^t(a,b).
\end{align*}
Similarly $\hat\sigma_b^t$ converges to $\hat\sigma_b$. However, we assumed that $\pi$ is not a policy equilibrium and thus no stationary distributions of $\M$, $\M_a$ and $\M_b$ can satisfy the policy equilibrium inequalities. We now arrive at a contradiction since by theorem~\ref{thm:conv_thm} and the above, we have that any limit point of $(\hat\sigma^t)_t$ and the corresponding distributions for fixed actions $a$ and $b$ are stationary distributions of $\M$, $\M_a$ and $\M_b$, respectively, which satisfy the policy equilibrium inequalities.
\end{proof}
Not surprisingly, we are able to show that the empirical distributions of observed play, also converge to the set of $S$ almost surely.
\begin{corollary}
\label{cor:main_cor}
The sequence of product distributions $(\tilde\sigma\times\tilde\sigma_a\times\tilde\sigma_b)_{T=1}^\infty$ converges weakly to the set $S$ almost surely.
\end{corollary}
\begin{proof}
We are going to show that for any convergent subsequence $(\hat\sigma^t)_t$ of $(\hat\sigma)_T$ with limit point $\sigma$, the corresponding subsequence $(\tilde\sigma^t)_t$ converges to $\sigma$ a.s.. Let $E$ be the event that $\lim_{t\rightarrow\infty}\|\sigma - \tilde\sigma^t\|_1 = 0$. Consider the following event $\lim_{t\rightarrow\infty}\|\sigma - \hat\sigma^t\|_1 + \lim_{t\rightarrow\infty}\|\hat\sigma^t - \tilde\sigma^t\|_1 = 0$ denoted by $E'$. Notice that every time $E'$ occurs $E$ also occurs because
\begin{align*}
\|\sigma - \tilde\sigma^t\|_1 \leq \|\sigma - \hat\sigma^t\|_1 + \|\hat\sigma^t - \tilde\sigma^t\|_1,
\end{align*}
and taking limits, preserves the non-strict inequality. By Theorem~\ref{thm:conc}, $E'$ occurs with probability $1$ and thus $E$ occurs w.p. 1 as well.
\end{proof}

\subsection{Proof that CCEs are a subset of policy equilibria}
\begin{lemma}
\label{lem:stable_ext_reg}
For any fixed game $\mathcal{G}$ and a CCE $\sigma$ of $\mathcal{G}$, there exist on average $(m,R(T))$ stable no-external regret algorithms, which if the players follow, the empirical distribution of play, converges to $\sigma$. Here $m< o(T)$ and $R(T) \leq O(\sqrt{T})$.
\end{lemma}
\begin{proof}
Fix the time horizon to be $T$. We assume that at round $t$ of the game, player $i$ has access to an oracle which provides $\hat\sigma_t$, the empirical distribution of play. The algorithm of player $i$ is the following -- split the time horizon into mini-batches of size $\sqrt{T}$. For the first mini-batch the player plays according to $\sigma$. At the end of the mini-batch, the oracle provides the player with $\hat\sigma_t$ and the player checks if $\|\hat\sigma_t - \sigma\|_1 \leq \frac{|\cA|}{T^{1/6}}$. If the condition is satisfied, then the algorithm continues playing according to $\sigma$. At the $j$-th epoch, the player checks if $\|\hat\sigma_t - \sigma\|_1 \leq \frac{|\cA|}{jT^{1/6}}$. If at any of the epochs the inequality turns out to be false, then the player switches to playing Exp3 at the end of the next epoch for the rest of the game i.e. if at the end of the $j$-th epoch, the inequality does not hold, then at the end of the $j+1$-st epoch, the player switches to playing Exp3. First we show that this is a no-external regret algorithm. Suppose until epoch $j$ the player has not switched to Exp3 and then at epoch $j$ the player switches. Let $u_t$ denote the utility function which the player observed at time $t$. Let $(a_t)_{t=1}^T$ be the action sequence of the player. For any fixed action $a\in\cA_i$ we have:
\begin{align*}
\expect{\sum_{t=1}^T u_t(a) - \sum_{t=1}^T u_t(a_t)} &= \expect{\sum_{t=1}^{(j-1)\sqrt{T}} u_t(a) - \sum_{t=1}^{(j-1)\sqrt{T}}  u_t(a_t)} + \sum_{t=(j-1)\sqrt{T}+1}^{(j+1)\sqrt{T}}\left[u_t(a) - u(a_t)\right]\\
&+ \expect{\sum_{t=(j+1)\sqrt{T}+1}^{T}u_t(a) - u(a_t)}\\
&\leq \expect{\sum_{t=1}^{(j-1)\sqrt{T}} u_t(a) - \sum_{t=1}^{(j-1)\sqrt{T}}  u_t(a_t)} + 2\sqrt{T} + c\sqrt{|\cA_i|\log{|\cA_i|}T}
\end{align*}
where in the first inequality we have used the fact that the utilities are in $[0,1]$ and the bound on the regret of Exp3 where $c$ is some constant. We now bound the term $\expect{\sum_{t=1}^{(j-1)\sqrt{T}} u_t(a) - \sum_{t=1}^{(j-1)\sqrt{T}}u_t(a_t)}$ via two applications of lemma~\ref{lem:distr_diff}.
\begin{align*}
&\frac{1}{T}\expect{\sum_{t=1}^{(j-1)\sqrt{T}} u_t(a) - \sum_{t=1}^{(j-1)\sqrt{T}}u_t(a_t)} = \expectation{(a',b')\sim\hat\sigma_t}{u_i(a,b')} - \expectation{(a',b')\sim\hat\sigma^t}{u_i(a',b')}\\
&=\expectation{(a',b')\sim\hat\sigma_t}{u_i(a,b')} - \expectation{(a',b')\sim\sigma}{u_i(a,b')} + \expectation{(a',b')\sim\sigma}{u_i(a,b')} - \expectation{(a',b')\sim\sigma}{u_i(a',b')}\\
&+ \expectation{(a',b')\sim\sigma}{u_i(a',b')} - \expectation{(a',b')\sim\hat\sigma_t}{u_i(a',b')}\\
&\leq \|\P_i\|\frac{(j-1)|\cA|}{jT^{1/6}} + \langle \P_i,\expectation{(a',b')\sim\sigma}{b'a'^\top}\rangle - \langle \P_i,\expectation{(a',b')\sim\hat\sigma_t}{b'a'^\top}\rangle\\
&\leq 2\|\P_i\|\frac{|\cA|}{T^{1/6}},
\end{align*}
where the first inequality follows from lemma~\ref{lem:distr_diff} and the fact $\sigma$ is a CCE. The second inequality again follows from lemma~\ref{lem:distr_diff}. The above implies $\expect{\sum_{t=1}^{(j-1)\sqrt{T}} u_t(a) - \sum_{t=1}^{(j-1)\sqrt{T}}u_t(a_t)}\leq \frac{\sqrt{T}|\cA|^2}{T^{1/6}}$. This follows fact that the spectral norm of the utility matrix for player $i$ can not exceed $|\cA|$, which again follows from the boundedness of utilities.

Next, we show that if all players, play according to the algorithm, then with high probability, the empirical distribution of play converges to $\sigma$. For this, we analyze the probability of the event that $\|\hat\sigma_t - \sigma\|_1 > \frac{|\cA|}{T^{1/6}}$ after the first epoch. By Azuma's inequality, we have that $\prob{|\hat\sigma_t(x) - \sigma(x)| > \frac{\epsilon}{|\cA|}} \leq 2\exp{-\epsilon^2t/(2|\cA|)}$, for any fixed action $x\in\cA$. By a union bound over all possible actions, we get that $\prob{\|\hat\sigma_1 - \sigma\|_1 > \epsilon} \leq 2|\cA|\exp{-\epsilon^2t/(2|\cA|)}$. Plugging in $\epsilon = \frac{|\cA|}{T^{1/6}}$, we get that the probability the algorithm fails after the first epoch is bounded by $2|\cA|\exp{-T^{1/6}|\cA|/2}$. Now the probability that the algorithm fails after epoch $j$ is bounded by 
\begin{align*}
\left(1-2|\cA|\exp{-T^{1/6}|\cA|/2}\right)^{j-1}2|\cA|\exp{-j^2T^{1/6}|\cA|/2} < 2|\cA|\exp{-T^{1/6}|\cA|/2}.
\end{align*}
A very pessimistic union bound now gives us that the probability the algorithm fails to converge to $\sigma$, provided that all players use it, is upper bounded by $2\sqrt{T}|\cA|\exp{-T^{1/6}|\cA|/2}$. 

Using a doubling trick, we can achieve the same results for any $T$. As in theorem~\ref{thm:conc}, we can now get $\lim\sup_{T\rightarrow\infty}\|\hat\sigma_T - \sigma\|_1 = 0$ almost surely and thus the empirical distribution of play converges to $\sigma$ a.s.. To check the stability of the algorithm we consider three possibilities -- the player is playing according to $\sigma$ and is not in an epoch after which they switch to Exp3, the player is playing according to $\sigma$ and is in an epoch after which they switch to Exp3 and the player has switched to playing Exp3. We take the view of player $1$. Consider the first case and suppose the current epoch of the algorithm is $j$. In particular consider time step $j\sqrt{T}+i$ at which the play was according to $\sigma_{j\sqrt{T}+i} = \sigma$ and suppose player $2$ in the last $m < o(\sqrt{T})$ switched from playing distributions $(p_t)_{t=j\sqrt{T}+i-m}^{j\sqrt{T}}$ to $(\tilde p_t)_{t=j\sqrt{T}+i-m}^{j\sqrt{T}}$ so that player $1$ plays $\tilde\sigma_{j\sqrt{T}+i}$. Since $m < o(\sqrt{T})$, the only change in the current algorithm can occur at the $j-1$-st epoch and is in the form of the algorithm decides to switch to Exp3. Since the switch can only occur after the $j$-th epoch is done, this implies that at time $j\sqrt{T}+i$ again the play is according to the distribution $\sigma$ and we have $\|\sigma_{j\sqrt{T}+i} - \tilde\sigma_{j\sqrt{T}+i}\|_1 = 0$. Consider the second case in which at the end of epoch $j$ the algorithm switches to playing Exp3. If the past $m$ actions of player $2$ change, this can only change the decision at the end of $j-1$-st epoch of player $1$ to switch to Exp3. However, as in the previous case, the player at the $j$-th epoch is still playing according to $\sigma$ and again the algorithm is stable at each iteration during the epoch. Finally, in the third case we have stability from the fact Exp3 is stable and once we enter an epoch during which Exp3 is played, changing the past $m$-actions of player $2$ can not change the decision of player $1$ to switch to Exp3, since $m<o(\sqrt{T})$ and the epochs are of size $\sqrt{T}$.
\end{proof}
As a consequence of the proof we see that, the algorithm is on average stable for any $m < o(\sqrt{T})$.

Theorems~\ref{lem_main:stable_ext_reg1} and~\ref{lem_main:stable_ext_reg2} are combined in the the following theorem.
\begin{theorem}
For any CCE $\sigma$ of a 2-player game $\mathcal{G}$, there exists a policy-equilibrium $\pi$, which induces a Markov chain $\M$, with stationary distribution $\sigma$. Further, there exists a 2-player game $\mathcal{G}$ and product distributions $\sigma\times\sigma_a\times\sigma_b \in S$, where $S$ is defined in~\ref{def:pol_eq_set}, such that $\sigma$ is not a CCE of $\mathcal{G}$.
\end{theorem}
\begin{proof}
First note that lemma~\ref{lem:stable_ext_reg} guarantees that for any CCE $\sigma$ there exists an $m$-stable on average no-external regret algorithm, which generates an empirical distribution of play $\tilde\sigma$ converging to $\sigma$ almost surely. From the $m$-stability of the algorithm and theorem~\ref{thm:stable}, we know that if both players play according to $\tilde\sigma$ then the empirical sequence of play also has no-policy regret and thus by theorem~\ref{thm:main_result}, we know that $\tilde\sigma\times\tilde\sigma_{a}\times\tilde\sigma_{b}$ converges to the set of policy-regret equilibria.

To show the second part of the theorem, we consider a game in which the pay-off matrix for player $2$ is just some constant multiple of $\mathbf{1}_{\cA_1}\mathbf{1}^\top_{\cA_2}$, where $\mathbf{1}_{\cA_i} \in \mathbb{R}^{|\cA_i|}$ is the vector with each entry equal to $1$. This implies that whatever sequence player $2$ chooses to play, they will have no-policy regret and no-external regret, since the observed utility in expectation is going to be equal for all actions $b\in\cA_2$. Let $\cA_1 = \{a,b\}$. The utility matrix of player $1$ is given in~\ref{util_matr}. 
\begin{table}[tbp]
\centering
\caption{Utility for player 1}
\label{util_matr}
\begin{tabular}{|l|l|l|}
\hline
player 1\textbackslash{}player 2 & \textbf{c\_1} & \textbf{c\_2} \\ \hline
\textbf{a}                       & 3/4           & 0             \\ \hline
\textbf{b}                       & 1             & 0             \\ \hline
\end{tabular}
\end{table}
The strategy of player $2$ is the following:
\begin{equation*}
f_{t}(a_{t-1}, a_{t}) = \begin{cases}
1 & a_{t-1} = a, a_{t} = b\\
\frac{3}{4} & a_{t-1} = a_{t} = a\\
0 & \text{otherwise}
\end{cases}.
\end{equation*}
Similarly to the proof of theorem~\ref{thm:reg_incomp}, we can show that any no-policy regret strategy for player $1$, incurs linear external regret. For any time step $t$, let $\hat\sigma^1_t$ be the dirac distribution for player $1$, putting all probability mass on action $a$ and let $\hat\sigma^2_t$ be the dirac distribution putting all probability mass on the action played by player $2$ at time $t$. Let $\hat\sigma = \frac{1}{T}\sum_{t=1}^T \hat\sigma^1_t(\hat\sigma^2_t)^\top$ be the empirical distribution after $T$ rounds. Then by theorem~\ref{thm:main_result} $\hat\sigma\times\hat\sigma_a\times\hat\sigma_b$ converges to the set of policy equilibria. However, by construction, if player $1$, plays according to $\hat\sigma$, then they will have linear external regret and thus the $\hat\sigma$ can not be converging to the set of CCEs.
\end{proof}

\subsection{Auxiliary results.}
\begin{lemma}
\label{lem:distr_diff}
Let $\sigma$ and $\sigma'$ be two distributions supported on a finite set and let $f$ be a utility/loss function uniformly bounded by $1$. If $\|\sigma - \sigma'\|_1 \leq \epsilon$ then $|\expectation{a\sim\sigma}{f(a)} - \expectation{a\sim\sigma'}{f(a)}| \leq \epsilon$.
\end{lemma}
\begin{proof}
\begin{align*}
&|\expectation{s\sim \sigma}{f(s)} - \expectation{s'\sim\sigma'}{f(s)}| = |\sum_{s\in S} \sigma(s)f(s) - \sum_{s\in S} \sigma'(s)f(s)|\\
&=|\sum_{s\in S}f(s)(\sigma(s)-\sigma'(s))|\leq \sum_{s\in S}|\sigma(s)-\sigma'(s)| = \|\sigma - \sigma'\|_1 \leq \epsilon.
\end{align*}
\end{proof}

\begin{lemma}
\label{lem:distr_close}
Let $\M \in\mathbb{R}^{d\times d}$ and $\hat\M\in\mathbb{R}^{d\times d}$ be two row-stochastic matrices, such that $\norm{\M - \hat\M} \leq \epsilon$, then for any stationary distribution $\hat\sigma$ of $\hat\M$, there exists a stationary distribution $\sigma$ of $\M$, such that $\|\hat\sigma - \sigma\|_1 \leq \frac{4d^2\epsilon}{\delta}$.
\end{lemma}
\begin{proof}
Let $\U \in \mathbb{R}^{d\times k}$ be the left singular vectors corresponding to the singular value $1$ of $\M$ and let $\hat\U \in \mathbb{R}^{d\times l}$ be the left singular vectors corresponding to the singular value $1$ of $\hat\M$. First notice that 
\begin{align*}\|\M\M^\top - \hat\M\hat\M^\top\| \leq \|\M\M^\top - \M\hat\M^\top\| + \|\M\hat\M^\top - \hat\M\hat\M^\top\| \leq (\|\M\| + \|\hat\M\|)\|\M - \hat\M\| \leq 2\epsilon
\end{align*}
Denote the eigen-gap of $\M$ by $\delta$, then by Wedin's theorem (see for example lemma B.3 in~\cite{allen2016lazysvd}) theorem we have 
\begin{align*}\|\hat\U^\top \U^\perp\| \leq \frac{\|\M\M^\top - \hat\M\hat\M^\top\|}{\delta} \leq \frac{2\epsilon}{\delta}.
\end{align*}
WLOG assume $\hat\sigma = \frac{(\hat\U)_i}{\|(\hat\U)_i\|_1}$. This implies that $\|\hat\sigma^\top\U^\perp\|_2 \leq \frac{2d\epsilon}{\delta}$ and thus:
\begin{align*}
\|\U\U^\top\hat\sigma - \hat\sigma\|_2 = \|(I-\U\U^\top)\hat\sigma\|_2 = \|\U^\perp(\U^\perp)^\top\hat\sigma\|_2 \leq \|\hat\sigma^\top\U^\perp\|_2 \leq \frac{2d\epsilon}{\delta}.
\end{align*}
Let $\sigma_i = \frac{\U_i}{\|\U_i\|_1}$ be the stationary distribution of $\M$, corresponding to the $i$-th left singular vector and let $\alpha_i = (\U^\top\hat\sigma)_i\|\U_i\|_1 \geq 0$. Then we have $\|\sum_i \alpha_i \sigma_i - \hat\sigma \|_1 \leq \frac{2d^2\epsilon}{\delta}$, where the inequality follows from the derivation above and the inequality between $l_1$ and $l_2$ norms. Let $\sigma = \frac{\sum_i \alpha_i \sigma_i}{\|\sum_i \alpha_i \sigma_i\|_1}$. This is a stationary distribution of $\M$, since 
\begin{align*}
\sigma^\top\M = \frac{1}{\|\sum_i \alpha_i \sigma_i\|_1}\sum_i \alpha_i\sigma_i^\top \M = \frac{\sum_i \alpha_i\sigma_i}{\|\sum_i \alpha_i \sigma_i\|_1} = \sigma.
\end{align*}
Notice that by reverse triangle inequality we have 
\begin{align*}
|\|\sum_i\alpha_i\sigma_i\|_1 - \|\hat\sigma\|_1| \leq \frac{2d^2\epsilon}{\delta},
\end{align*}
or equivalently 
\begin{align*}
|\|\sum_i\alpha_i\sigma_i\|_1 - 1| \leq \frac{2d^2\epsilon}{\delta}.
\end{align*}
Thus we have:
\begin{align*}
\|\sigma - \hat\sigma\|_1 \leq \|\sum_i\alpha_i\sigma_i - \hat\sigma\|_1 + \|\sigma - \sum_i\alpha_i\sigma_i - \hat\sigma\|_1 \leq \frac{2d^2\epsilon}{\delta} + \|\sigma\|_1|1-\|\sum_i\alpha_i\sigma_i\|_1| \leq \frac{4d^2\epsilon}{\delta}.
\end{align*}
\end{proof}

\begin{corollary}
Let the empirical distribution of observed play be $\tilde\sigma^T = \frac{1}{T}\sum_{t=1}^T \delta_t$, the empirical distribution of play if player $1$ deviated to playing fixed action $a\in\cA_1$ be $\tilde\sigma^T_a$ and the empirical distribution of play if player $2$ to action $b\in\cA_2$ be $\tilde\sigma^T_b$. The sequence $(\tilde\sigma^T,\tilde\sigma^T_a,\tilde\sigma^T_b)_T$ converges to the set $P$ almost surely.
\end{corollary}
\begin{proof}
Lemma~\ref{lem:distr_diff}, together with theorem~\ref{thm:conc} imply that \begin{align*}
\lim\sup_{T\rightarrow\infty}|\expectation{(a,b)\sim\tilde\sigma^T}{u_1(a,b)} - \expectation{(a,b)\sim\hat\sigma^T}{u_1(a,b)}| = 0
\end{align*}almost surely i.e. $$\prob{\lim\sup_{T\rightarrow\infty}|\expectation{(a,b)\sim\tilde\sigma^T}{u_1(a,b)} - \expectation{(a,b)\sim\hat\sigma^T}{u_1(a,b)}| = 0} = 1.$$ Since $$\lim\sup_{T\rightarrow\infty}|\expectation{(a,b)\sim\tilde\sigma^T}{u_1(a,b)} - \expectation{(a,b)\sim\hat\sigma^T}{u_1(a,b)}| \geq \lim\inf_{T\rightarrow\infty}|\expectation{(a,b)\sim\tilde\sigma^T}{u_1(a,b)} - \expectation{(a,b)\sim\hat\sigma^T}{u_1(a,b)}| \geq 0,$$ this implies that $$\prob{\lim\inf_{T\rightarrow\infty}|\expectation{(a,b)\sim\tilde\sigma^T}{u_1(a,b)} - \expectation{(a,b)\sim\hat\sigma^T}{u_1(a,b)}| = 0} = 1.$$ On the other hand this implies $$\prob{\lim\inf_{T\rightarrow\infty}|\expectation{(a,b)\sim\tilde\sigma^T}{u_1(a,b)} - \expectation{(a,b)\sim\hat\sigma^T}{u_1(a,b)}| = 0 \bigcap \lim\sup_{T\rightarrow\infty}|\expectation{(a,b)\sim\tilde\sigma^T}{u_1(a,b)} - \expectation{(a,b)\sim\hat\sigma^T}{u_1(a,b)}| = 0} \geq 1$$ and so $$\prob{\lim_{T\rightarrow\infty}|\expectation{(a,b)\sim\tilde\sigma^T}{u_1(a,b)} - \expectation{(a,b)\sim\hat\sigma^T}{u_1(a,b)}| = 0} = 1.$$ In a similar way we can get $\lim_{T\rightarrow\infty}|\expectation{(a,b)\sim\tilde\sigma^T_a}{u_1(a,b)} - \expectation{(a,b)\sim\hat\sigma^T_a}{u_1(a,b)}| = 0$ a.s. The above imply that $\lim_{T\rightarrow\infty}\expectation{(a,b)\sim\tilde\sigma^T}{u_1(a,b)} - \expectation{(a,b)\sim\hat\sigma^T}{u_1(a,b)} = 0$ a.s. and $\lim_{T\rightarrow\infty} \expectation{(a,b)\sim\hat\sigma^T_a}{u_1(a,b)} - \expectation{(a,b)\sim\tilde\sigma^T_a}{u_1(a,b)} = 0$ a.s. and thus:
\begin{align*}
0 &= \lim_{T\rightarrow\infty}\expectation{(a,b)\sim\tilde\sigma^T}{u_1(a,b)} - \expectation{(a,b)\sim\hat\sigma^T}{u_1(a,b)} + \lim_{T\rightarrow\infty} \expectation{(a,b)\sim\hat\sigma^T_a}{u_1(a,b)} - \expectation{(a,b)\sim\tilde\sigma^T_a}{u_1(a,b)}\\
&=\lim_{T\rightarrow\infty}\expectation{(a,b)\sim\tilde\sigma^T}{u_1(a,b)} - \expectation{(a,b)\sim\tilde\sigma^T_a}{u_1(a,b)} + \lim_{T\rightarrow\infty}\expectation{(a,b)\sim\hat\sigma^T_a}{u_1(a,b)} - \expectation{(a,b)\sim\hat\sigma^T}{u_1(a,b)}
\end{align*}
a.s.. Since $\expectation{(a,b)\sim\hat\sigma^T_a}{u_1(a,b)} - \expectation{(a,b)\sim\hat\sigma^T}{u_1(a,b)} < o(1)$, this implies that 
\begin{align*}
0 &\leq - \lim_{T\rightarrow\infty}\expectation{(a,b)\sim\hat\sigma^T_a}{u_1(a,b)} - \expectation{(a,b)\sim\hat\sigma^T}{u_1(a,b)}\\
&=\lim_{T\rightarrow\infty}\expectation{(a,b)\sim\tilde\sigma^T}{u_1(a,b)} - \expectation{(a,b)\sim\tilde\sigma^T_a}{u_1(a,b)}
\end{align*}
a.s.. Now we can proceed as in the proof of theorem~\ref{thm:conv_thm}.
\end{proof}

\subsection{Concentration of $\tilde\M$}
\label{sec:conc_of_matrices}
\begin{lemma}
\label{lem:mart_bound}
With probability at least $1 - |\cA|6\exp{-\frac{T\epsilon^2}{4}}$ it holds that $|\frac{1}{T}\sum_{t=1}^T p_{t-1}(x_i)p_t(x_j) - \frac{1}{T}\sum_{t=1}^T \delta_{t-1}(x_i)\delta_t(x_j)| < \epsilon$ and $|\frac{1}{T}\sum_{t=1}^T p_t(x_i) - \frac{1}{T}\sum_{t=1}^T \delta_t(x_i)|<\epsilon$, simultaneously for all $i$. 
\end{lemma}
\begin{proof}
We consider the random variable $Z_t = \delta_t(x_i) - p_t(x_i)$, notice that $\expect{Z_t|p_1,\cdots,p_{t-1}} = 0$ so that $\{Z_t\}_t$ is a bounded martingale sequence with $|Z_t| < 1$ and thus by Azuma's inequality we have $\prob{\left|\frac{1}{T}\sum_{t=1}^T Z_t\right| \geq \epsilon} < 2\exp{-\frac{T\epsilon^2}{2}}$ which shows that $\frac{1}{T}\sum_{t=1}^T \delta_t(a_i)$ concentrates around $\frac{1}{T}\sum_{t=1}^T p_t(x_i)$. Let $R_t = \delta_{t-1}(x_i)\delta_{t}(x_j) - p_{t-1}(x_i)p_{t}(x_j)$ and consider the filtration $\{\cF_t\}_{t}$, where $\cF_1 = \emptyset$, $\cF_{t} = \Sigma(\delta_1,\cdots,\delta_{t})$ is the sigma algebra generated by the random variables $\delta_1$ to $\delta_{t}$. Then $|R_{2t}| \leq 1$ and $\expect{R_{2t}|\cF_1,\cdots,\cF_{2t-2}} = 0$, so $\{R_{2t}\}_t$ is also a bounded martingale difference and thus $\prob{\left|\frac{1}{T}\sum_{t=1}^\frac{T}{2} R_{2t}\right| \geq \frac{\epsilon}{2}} < 2\exp{-\frac{T\epsilon^2}{4}}$. A similar argument allows us to bound the sum of the $R_{2t+1}$'s and a union bound gives us $\prob{\left|\frac{1}{T}\sum_{t=1}^{T} R_{t}\right| \geq \epsilon} < 4\exp{-\frac{T\epsilon^2}{4}}$. A union bound over all $i$ finishes the proof.
\end{proof}

\begin{definition}
\label{def:perturbed}
Define the perturbed distribution of player $i$ at time $t$ to be $\tilde p_t^i = (1-\sqrt{|\cA|\tilde\epsilon})p_t^i + \textbf{1}\frac{\sqrt{|\cA|\tilde\epsilon}}{|\cA_i|}$.
\end{definition}

\begin{lemma}
The difference of expected utilities from playing according to $(\tilde p_t^i)_{t=1}^T$ instead of $(p_t^i)_{t=1}^T$ is at most $2T\sqrt{|\cA|\tilde\epsilon}$
\end{lemma}
\begin{proof}
From lemma~\ref{lem:distr_diff} at each time step the difference of expected utility is bounded by $\sqrt{|\cA|\tilde\epsilon}$ in absolute value.
\end{proof}

\begin{theorem}
\label{thm:conc}
If at time $t$ player $i$ plays according to $\tilde p_t^i$ , where $\tilde\epsilon = \frac{T^{-1/4}}{|\cA|}$ and $p_t = \tilde p_t^1 (\tilde p_t^2)^\top$, then the regret for playing according to $\tilde p_t^i$ is at most $O(T^{7/8})$. Further $\lim\sup_{T\rightarrow\infty} \norm{\tilde\M - \hat\M}_2 = 0$, almost surely. Additionally if $\tilde\sigma = \frac{1}{T}\sum_{t=1}^T \delta_t$ is the stationary distribution of $\tilde\M$ corresponding to the observed play and $\hat\sigma = \frac{1}{T}\sum_{t=1}^T p_t$ is the stationary distribution of $\hat\M$ corresponding to the averaged empirical distribution, then $\lim\sup_{T\rightarrow\infty}\norm{\tilde\sigma - \hat\sigma}_1 = 0$ almost surely.
\end{theorem}
\begin{proof}
Set $\tilde\epsilon = \frac{T^{-1/4}}{|\cA|}$. The regret bound of the no-external regret algorithms now becomes $O(T^{7/8})$. We can, however, now guarantee that $\sum_{t=1}^T p_t(x_i) \geq \frac{T^{-1/4}}{|\cA|}$ and thus, combining this with the high probability bound we obtain that with probability at least $1-C\exp{-\frac{T\epsilon^2}{4}}$ it holds that $|\tilde \M_{i,j} - \hat \M_{i,j} | < 2\epsilon\frac{T^{1/4}}{|\cA|}$. To see this, let $x = \frac{1}{T}\sum_{t=1}^T p_{t-1}(x_i)p_t(x_j),\hat x = \frac{1}{T}\sum_{t=1}^T \delta_{t-1}(x_i)\delta_t(x_j), y = \frac{1}{T}\sum_{t=1}^T p_{t}(x_i), \hat y = \frac{1}{T}\sum_{t=1}^T \delta_{t}(x_i)$. Then \begin{align*}
|\tilde \M_{i,j} - \hat \M_{i,j}| = \left|\frac{x}{y} - \frac{\hat x}{\hat y}\right| \leq \frac{|x-\hat x|}{|y|} + \frac{|\hat x|}{|1/y - 1/\hat y|} \leq \frac{\epsilon}{\tilde\epsilon} + \frac{|\hat x||y - \hat y|}{|y \hat y|} \leq 2\frac{\epsilon}{\tilde\epsilon},
\end{align*}
where the last inequality holds because $\hat x \leq \hat y$. Setting $\epsilon = T^{-1/3}$ and a union bound we arrive at $\prob{\norm{\tilde \M - \hat\M}_2 > T^{-1/12}} < C\exp{-\frac{T^{1/3}}{4}}$. By Borel-Cantelli lemma we have $\lim\sup_{T\rightarrow\infty} \norm{\tilde\M - \hat\M}_2 = 0$ almost surely. From lemma~\ref{lem:mart_bound} and a union bound we know that with $\prob{\|\tilde\sigma - \hat\sigma\|_1 > \epsilon} < 2|\cA|\exp{-\frac{T\epsilon^2}{4}}$. Setting $\epsilon = T^{-1/3}$ and again using Borel-Cantelli's lemma we see that $\lim\sup_{T\rightarrow\infty}\norm{\tilde\sigma - \hat\sigma}_1 = 0$.
\end{proof}

\section{Extending the framework to arbitrary memories}
\begin{definition}
Let player $1$ have memory $m_1$ and player $2$ have memory $m_2$. Let the function spaces be $\cF_1 = \{f:\cA_2^{m_1} \rightarrow \cA_1\}$ and $\cF_2 = \{f:\cA_1^{m_2} \rightarrow \cA_2\}$. Let $\pi$ be a distribution over $\cF_1\times\cF_2$. Define the $m$-memory bounded Markov process, where $m = \max(m_1,m_2)$ to be $\prob{(a_t,b_t) | (a_{t-1},b_{t-1}), \cdots, (a_{t-m},b_{t-m})} = \sum_{(f,g)\in\cF_1\times\cF_2: f(b_{t-1},\cdots,b_{t-m_1}) = a_t, g(a_{t-1},\cdots,a_{t-m_2}) = b_{t}} \pi(f,g)$. We associate with this Markov process the matrix $\M \in \mathbb{R}^{|\cA|^{m}\times|\cA|^{m}}$, with $\M_{(x_{t-1},\cdots,x_{t-m}),(x_t,x_{t-1},\cdots,x_{t-m+1})} = \prob{x_t | x_{t-1},\cdots,x_{t-m}}$ and $\M_{i,j} = 0$ for all other entries.
\end{definition}
\begin{definition}
The utility of $\pi$ is defined through the stationary distribution $\gamma$ of $\M$. In particular $\gamma$ is a distribution over $|\cA|^{m}$ with entries indexed by $(x_t,\cdots,x_{t-m+1})$. Let $\sigma$ be the marginal distribution of $x_t$, then $u_i(\pi) = \sup_{\sigma}\expectation{(a,b)\sim\sigma}{u_i(a,b)}$
\end{definition}
\begin{definition}
The empirical $m$-memory bounded Markov process is $\hat\M$ with $\prob{x_1| x_2,\cdots,\x_{m+1}} = \frac{\sum_t \prod_{i=0}^{m-1} p_{t-i}(x_{i+1})\times \prod_{i=0}^{m-1} p_{t-i}(x_{i+2}) }{\sum_t\prod_{i=0}^{m-1} p_{t-i}(x_{i+2})}$
\end{definition}
\begin{theorem}
The distribution $\hat\sigma$ over $|\cA|^{m}$ with entries indexed as $$\hat\sigma((a_1,b_1),\cdots,(a_m,b_m)) = \frac{1}{T}\sum_{t} \prod_{i=0}^{m-1} p_{t-i}((a_{i+1},b_{i+1}))$$, is a stationary distribution of $\hat\M$.
\end{theorem}
\begin{proof}
Consider $(\hat\sigma^\top\hat\M)_{x_1,\cdots,x_{m}}$, we show it is equal to $\hat\sigma(x_1,\cdots,x_{m})$:
\begin{align*}
    (\hat\sigma^\top\hat\M)_{x_1,\cdots,x_{m}} &= \sum_{x_2,\cdots,x_{m+1}}\hat\M_{(x_2,\cdots,x_{m+1}),(x_1,\cdots,x_{m})}\hat\sigma(x_2,\cdots,x_{m+1})\\
    &=\sum_{x_2,\cdots,x_{m+1}} \prob{x_{1} | x_2,\cdots,x_{m+1}}\hat\sigma(x_2,\cdots,x_{m+1})\\
    &= \frac{1}{T}\sum_{x_2,\cdots,x_{m+1}}\frac{\sum_t \prod_{i=0}^{m-1} p_{t-i}(x_{i+1})\times \prod_{i=0}^{m-1} p_{t-i}(x_{i+2}) }{\sum_t\prod_{i=0}^{m-1} p_{t-i}(x_{i+2})}\sum_t \prod_{i=0}^{m-1}p_{t-i}(x_{m-i+2})\\
    &= \frac{1}{T}\sum_{x_2,\cdots,x_{m+1}}\sum_t \prod_{i=0}^{m-1} p_{t-i}(x_{i+1})\times \prod_{i=0}^{m-1} p_{t-i}(x_{i+2})\\
    &=\frac{1}{T}\sum_t \prod_{i=0}^{m-1} p_{t-i}(x_{i+1}) = \hat\sigma(x_1,\cdots,x_m).
\end{align*}
\end{proof}
We also need a theorem which states that the utility $\expectation{(f,g)\sim\pi}{u_1(a,g(a,\cdot,a)}$ is the expectation of the stationary distribution of the Markov process coming from the play $(a,g(a,\cdots,a))$ according to the marginal of $g$.
\begin{definition}
Let $\M_1^a$ be the $m_2$-memory bounded Markov process which comes from player $1$ playing a fixed action $a\in\cA_1$ and player $2$ playing $g\in\cF_2$ according to the marginal of $\pi$, i.e. $(\M_1^a)_{(x_1,\cdot,x_{m_2}),(x_1,\cdot,x_{m_2+1})} = \prob{x_{m_2+1}|x_1,\cdots,x_{m_2}} = \sum_{(f,g): g(a_1,\cdots,a_{m_2}) = b_{m_2+1}}\pi(f,g)$ if $a_{m_2+1} = a$ or $0$ otherwise (here $x_i = (a_i,b_i)$). Similarly let $\M_2^b$ be the $m_1$-memory bounded Markov process which arises when player $2$ switches to playing the constant action.
\end{definition}
\begin{theorem}
Consider $\M_1^a$ and let $m = m_2$. Let $\bar\sigma$ be the following distribution over $\cA^m$ -- $\bar\sigma(x_1,\cdots,x_m) = \sum_{(f,g): g(a,\cdot,a) = b} \pi(f,g)$ if $x_1=\cdots=x_m = (a,b)$ and $0$ otherwise. Then $\bar\sigma$ is a stationary distribution of $\M_1^a$ and its marginal distribution $\gamma(a,b) = \sum_{(f,g): g(a,\cdot,a) = b} \pi(f,g)$ is such that $\expectation{(a,b)\sim\gamma}{u_1(a,b)} = \expectation{(f,g)\sim\pi}{u_1(a,g(a,\cdots,a))}$.
\end{theorem}
\begin{proof}
\begin{align*}(\bar\sigma^\top\M_1^a)_{x_2,\cdots,x_{m+1}} &= \sum_{x_1} \bar\sigma(x_1,\cdots,x_m)(\M_1^a)_{(x_1,\cdots,x_m),(x_2,\cdots,x_{m+1})}\\
&= \sum_{x_1}\bar\sigma(x_1,\cdots,x_m)\prob{x_{m+1}|x_1,\cdots,x_m}\\
&= \sum_{b_1\in\cA_2}\sigma((a,b_1),\cdots,(a,b_1))\sum_{f,g(a,\cdots,a) = b_{m+1}}\pi(f,g)\\
&= \sum_{b_1\in\cA_2}\left[\sum_{f,g(a,\cdots,a) = b_1}\pi(f,g)\right]\sum_{f,g(a,\cdots,a) = b_{m+1}}\pi(f,g)\\
&= \sum_{f,g(a,\cdots,a)=b_{m+1}}\pi(f,g)
= \sigma((a,b_{m+1}),\cdots,(a,b_{m+1})) = \sigma(x_2,\cdots,x_{m+1})
\end{align*}
\end{proof}
The rest of the proofs extend in similar ways to the case with general memory. There is the question of not requiring $\bar\sigma$ to be a stationary distribution but to be any distribution such that the marginal with respect to the first coordinate of $\bar\sigma^\top\M_1^a$, satisfies the expectation equality $\expectation{(a,b)\sim\gamma}{u_1(a,b)} = \expectation{(f,g)\sim\pi}{u_1(a,g(a,\cdots,a))}$.
\section{Extending the framework to arbitrary number of players}
\begin{definition}
Consider an $n$ player game where player $i$ has memory $m_i$. Define the set of functions $\cF_i = \{f:\cA_{-i}^{m_i} \rightarrow \cA_i\}$. We consider a distribution $\pi$ over $\cF = \times_{i=1}^n \cF_i$. Let $m = \max(m_1,\cdots,m_n)$. $Let x = (a^1,\cdots,a^n) \in \cA$ and let $x^{-i} = (a^1,\cdots,a^{i-1},a^{i+1},\cdots,a^n) \in \cA_{-i}$. Define the $m$-memory Markov process $\M$ such that $\prob{x_{m+1}|x_1,\cdots,x_m} = \sum_{(f_1(x^{-1}_m,\cdots,x^{-1}_{m-m_1+1}) = a^1_{m+1},\cdots,f_n(x^{-n}_m,\cdots,x^{-n}_{m-m_n+1}) = a^n_{m+1})}\pi(f_1,\cdots,f_n)$.
\end{definition}
All other definitions follow the same form. The problem with this most general setting is -- how do we construct no-policy regret algorithms and what does it even mean to have no-policy regret? The utility function of player $i$ at time $t$ -- $u_i(\cdot,x_{t}^{-i})$ is no longer interpretable as an $m$-memory bounded function, since $x_{t}^{-i}$ depends on all other players' memories. One possible solution is to look at this utility as an $m$-memory bounded function, where $m$ is the maximum memory among all other players.

\section{Simple example of a policy equilibrium}
\label{sec:simp_example}
We now present a simple 2-player game with strategies of the players which lead to a policy equilibrium, which in fact is not a CCE. Further these strategies give the asymptotically maximum utility for both row and column players over repeated play of the game. The idea behind the construction is very similar to the one showing incompatibility of policy regret and external regret. 
\begin{table}[h]
\centering
\caption{Utility matrix}
\label{game_matr}
\begin{tabular}{|l|l|l|}
\hline
Player 1\textbackslash{}Player 2& \textbf{c} & \textbf{d} \\ \hline
\textbf{a}                       & (3/4,1)           & (0,1)             \\ \hline
\textbf{b}                       & (1,1)             & (0,1)             \\ \hline
\end{tabular}
\end{table}
The utility matrix for the game is given in Table~\ref{game_matr}. Since the column player has the same payoff for all his actions they will always have no policy and no external regret. The strategy the column player chooses is to always play the function $f:\cA_1 \rightarrow \cA_2$:
\begin{equation*}
f(x) = \begin{cases}
c & x = a\\
d & x = b.
\end{cases}
\end{equation*}
In the view of the row player, this strategy corresponds to playing against an adversary which plays the familiar utility functions:
\begin{equation*}
u_{t}(a_{t-1}, a_{t}) = \begin{cases}
1 & a_{t-1} = a, a_{t} = b\\
\frac{3}{4} & a_{t-1} = a_{t} = a\\
0 & \text{otherwise}.
\end{cases}
\end{equation*}
We have already observed that on these utilities, the row player can have either no policy regret or no external regret but not both. What is more the utility of no policy regret play is higher than the utility of any of the no external regret strategies. This already implies that the row player is better off playing according to the no policy regret strategy which consists of always playing the fixed function $g:\cA_2 \rightarrow \cA_1$ given by $g(x) = a$. Below we present the policy equilibrium $\pi \in \Delta \cF$, corresponding Markov chain $\M \in \mathbb{R}^{|\cA|\times|\cA|}$ and its stationary distribution $\sigma \in \Delta \cA$ satisfying the no policy regret requirement.
\begin{align*}
    \pi(\tilde f,\tilde g) = \delta_{(f,g)}, 
    \M = \begin{blockarray}{ccccc}
    &(a,c) & (a,d) & (b,c) & (b,d)\\
    \begin{block}{c(cccc)}
    (a,c) & 1 & 0 & 0 & 0\\
    (a,d) & 1 & 0 & 0 & 0\\
    (b,c) & 0 & 1 & 0 & 0\\
    (b,d) & 0 & 1 & 0 & 0\\
    \end{block}
    \end{blockarray},
    \sigma(x,y) = \delta_{(a,c)}.
\end{align*}
Suppose the row player was playing any no policy regret strategy, for example one coming from a no policy regret algorithm, as a response to the observed utilities $u_t(\cdot,\cdot)$. Since the only sublinear policy regret play for these utilities is to only deviate from playing $a$ a sublinear number of times we see that the empirical distribution of play for the row player converges to the dirac distribution $\delta_a$. Together with the strategy of the column player, this implies the column player chooses the action $d$ only a sublinear number of times and thus their empirical distribution of play converges to $\delta_c$. It now follows that the empirical distribution of play converges to $\delta_a\times\delta_c = \delta_{(a,c)} \in \Delta \cA$. We can similarly verify that the empirical Markov chain will converge to $\M$ and the empirical functional distribution $\hat\pi$ converges to $\pi$. Theorem~\ref{thm:main_result} guarantees that because both players incur only sublinear regret $\pi$ is a policy equilibrium. It should also be intuitively clear why this is the case without the theorem -- suppose that the row player switches to playing the fixed action $b$. The resulting functional distribution, Markov chain and stationary distributions become:
\begin{align*}
    \pi_b(\tilde f,\tilde g) = \delta_{(f,\hat g\equiv b)}, 
    \M_b = \begin{blockarray}{ccccc}
    &(a,c) & (a,d) & (b,c) & (b,d)\\
    \begin{block}{c(cccc)}
    (a,c) & 0 & 0 & 1 & 0\\
    (a,d) & 0 & 0 & 1 & 0\\
    (b,c) & 0 & 0 & 0 & 1\\
    (b,d) & 0 & 0 & 0 & 1\\
    \end{block}
    \end{blockarray},
    \sigma_b(x,y) = \delta_{(b,d)}.
\end{align*}
The resulting utility for the row player is now $0$, compared to the utility gained from playing according to $\pi$, which is $3/4$.

\end{document}